\documentclass[11pt,reqno]{amsart}

\usepackage{amsmath, amsthm, amssymb}
\usepackage{amsfonts}
\usepackage{bbm}
\usepackage[ansinew]{inputenc}
\usepackage[dvips]{epsfig}
\usepackage{graphicx}
\usepackage[english]{babel}
\usepackage{enumerate}
\usepackage{hyperref}
\usepackage{fbox}
\usepackage{mathrsfs}
\usepackage{color}
 \usepackage{accents}
 \usepackage{subcaption}
\usepackage{graphicx}
\definecolor{cb-red}   {RGB}{223,   80,  80}

\newlength{\dhatheight}
\newcommand{\doublehat}[1]{%
    \settoheight{\dhatheight}{\ensuremath{\hat{#1}}}% 
    \addtolength{\dhatheight}{-0.35ex}%
    \hat{\vphantom{\rule{1pt}{\dhatheight}}%
    \smash{\hat{#1}}}}

\theoremstyle{plain}
\newtheorem{thm}{Theorem}[section]

\newtheorem{lem}[thm]{Lemma}
\newtheorem{prop}[thm]{Proposition}

\theoremstyle{definition}
\newtheorem{defi}[thm]{Definition}

\theoremstyle{remark}
\newtheorem{rem}[thm]{Remark}

\numberwithin{equation}{section}

\newcommand\bovermat[2]{%
  \makebox[0pt][l]{$\smash{\overbrace{\phantom{%
    \begin{matrix}#2\end{matrix}~~}}^{\text{#1}}}$}#2}

\newcommand{\R}{\mathbb{R}}

\newcommand{\N}{\mathbb{N}}
\newcommand{\PP}{\mathbb{P}}

\newcommand{\eps}{\varepsilon}

\newcommand{{\barV}}{{\overline{{V}}}}
\newcommand{\bv}{{\boldsymbol{{v}}}}
\newcommand{\bw}{{\boldsymbol{{w}}}}
\newcommand{\bXi}{{\boldsymbol{{\Xi}}}}

\newcommand{\bZ}{\boldsymbol{Z}}
\newcommand{\bG}{\boldsymbol{G}}
\newcommand{\bU}{\boldsymbol{U}}
\newcommand{\bV}{\boldsymbol{V}}
\newcommand{\bW}{\boldsymbol{W}}
\newcommand{\bPsi}{\boldsymbol{\Psi}}
\newcommand{\be}{{\boldsymbol{{e}}}}
\newcommand{\bl}{{\boldsymbol{{\lambda}}}}

\newcommand{\brU}{{\mathring{\boldsymbol{U}}}}
\newcommand{\brV}{{\mathring{\boldsymbol{V}}}}
\newcommand{\brW}{{\mathring{\boldsymbol{W}}}}

\newcommand{\brp}{{\mathring{\boldsymbol{p}}}}
\newcommand{\brq}{{\mathring{\boldsymbol{q}}}}

\newcommand{\bJ}{\boldsymbol{J}}
\newcommand{\bGamma}{\boldsymbol{\Gamma}}
\newcommand{\bK}{\boldsymbol{K}}
\newcommand{\bR}{\boldsymbol{R}}
\newcommand{\bS}{\boldsymbol{S}}
\newcommand{\w}{\boldsymbol{w}}
\newcommand{\bxi}{\boldsymbol{\xi}}
\newcommand{\brxi}{{\mathring{\boldsymbol{\xi}}}}
\newcommand{\B}{\boldsymbol{B}}
\newcommand{\balpha}{\boldsymbol{\alpha}}
\newcommand{\bbeta}{\boldsymbol{\beta}}
\newcommand{\bgamma}{\boldsymbol{\gamma}}
\newcommand{\A}{\boldsymbol{A}}
\newcommand{\M}{\boldsymbol{M}}
\newcommand{\bp}{\boldsymbol{p}}
\newcommand{\0}{\boldsymbol{0}}
\newcommand{\LL}{\boldsymbol{\Lambda}}

\newcommand{\E}{\mathbb{E}}

\newcommand{\G}{\boldsymbol{G}}
\newcommand{\bq}{\boldsymbol{q}}
\newcommand{\bO}{\boldsymbol{O}}

\newcommand{\bX}{\boldsymbol{X}}
\newcommand{\z}{\boldsymbol{z}}

\newcommand{\bN}{\boldsymbol{N}}
\newcommand{\bP}{\boldsymbol{P}}
\newcommand{\bPi}{\boldsymbol{\Pi}}

\newcommand{\kU}{\boldsymbol{\mathfrak{U}}}
\newcommand{\kV}{\boldsymbol{\mathfrak{V}}}
\newcommand{\kW}{\boldsymbol{\mathfrak{W}}}
\newcommand{\kE}{\boldsymbol{\mathfrak{E}}}

\newcommand{\mkU}{{\mathfrak{U}}}
\newcommand{\mkV}{{\mathfrak{V}}}
\newcommand{\mkW}{{\mathfrak{W}}}

\newcommand{\average}{{\mathchoice {\kern1ex\vcenter{\hrule height.4pt
width 6pt depth0pt} \kern-9.7pt} {\kern1ex\vcenter{\hrule
height.4pt width 4.3pt depth0pt} \kern-7pt} {} {} }}

\def\R{\mathbb{R}}

\author{L\'ena\"ic Chizat}
\address{EPFL SB MATH, Institute of Mathematics, Station 8, CH-1015 Lausanne, Switzerland}
\email{lenaic.chizat@epfl.ch}

\author{Maria Colombo}
\address{EPFL SB MATH, Institute of Mathematics, Station 8, CH-1015 Lausanne, Switzerland}
\email{maria.colombo@epfl.ch}

\author{Xavier Fern\'andez-Real}
\address{EPFL SB MATH, Institute of Mathematics, Station 8, CH-1015 Lausanne, Switzerland}
\email{xavier.fernandez-real@epfl.ch}

\author{Alessio Figalli}
\address{ETH Zurich, Department of Mathematics, R\"amistrasse 101, 8092 Z\"urich, Switzerland}
\email{alessio.figalli@math.ethz.ch}

\title[Infinite-width limit of deep linear neural networks]{Infinite-width limit of \\deep linear neural networks}

\thanks{M. C. and X. F. were supported by the SNF grant 200021\_182565 and by the Swiss State Secretariat for Education Research and Innovation (SERI) under contract number M822.00034. X. F. was furthermore supported by the SNF grant PZ00P2\_208930. A. F. was supported by the European Research Council (ERC) under grant agreement No 721675 ``Regularity and Stability in Partial Differential Equations (RSPDE)'' and by the Lagrange Mathematics and Computation Research Center.}
\subjclass[2020]{68T07, 35Q49.}

\begin{document}

\begin{abstract}
This paper studies the infinite-width limit of deep linear neural networks initialized with random parameters. We obtain that, when the number of neurons diverges, the training dynamics converge (in a precise sense) to the dynamics obtained from a gradient descent on an infinitely wide deterministic linear neural network. Moreover, even if the weights remain random, we get their precise law along the training dynamics, and prove a quantitative convergence result of the linear predictor in terms of the number of neurons. 

We finally study the continuous-time limit obtained for infinitely wide linear neural networks and show that the linear predictors of the neural network converge at an exponential rate to the minimal $\ell_2$-norm minimizer of the risk. 
\end{abstract}

\maketitle

\section{Introduction}
The description of the training dynamics of (artificial) neural networks (NNs) in the infinite-width limit, has in   recent years shed light on several aspects of deep learning theory, such as (i) the existence of well-posed limits, which suggests to interpret practical large scale models as approximations of those limits, (ii) the importance of the choice of scalings/parametrization\footnote{That is, the choice, as a function of the width, of the variance of the random initialization and of the learning rates for each layer.} when passing to the limit --- since several well-behaved but fundamentally different limits can be obtained, and (iii) the characterization of the long-term behavior of the dynamics --- such as global convergence or algorithmic regularization --- which in turn helps understanding the learning abilities of neural networks.

These aspects are rather well understood for two-layer neural networks, but the theory is lacunary for deeper NNs. A description of the infinite-width dynamics is available for the Neural Tangent (NTP) and Integrable (IP) parameterizations (discussed below), but both limits exhibit a form of degeneracy such as a lack of feature learning. In~\cite{yang2021tensor}, the Maximal Update Parameterization ($\mu$P) --- which is in a sense intermediate between (NTP) and (IP) in terms of scale --- was introduced and shown to preserve feature learning in the limit for certain architectures, such as fully-connected NNs, which suggests that $\mu$P is a natural case of study. However, the theoretical understanding of this limit is so far very limited, because this limit involves large random matrices in an intricate way. In particular, the following fundamental questions  are still open: 
\begin{itemize}
\item Is the infinite-width limit of Gradient Descent (GD) a GD trajectory in some infinite-dimensional space? 
\item Does it admit a well-posed continuous-time limit? 
\item Does it converge to minimizers? And when several minimizers exist, can we characterize which particular solution it selects? 
\end{itemize}

In this paper, we study the infinite-width limit of deep \emph{linear}\footnote{Linear NNs are NNs without nonlinear maps between layers. Although they are linear in the input data, we note that these models are non-linear in their parameters.} NNs under $\mu$P, and we answer positively to all these questions.  Throughout the paper, we focus on the three-layer case, although our tools and analysis could be extended to more layers (the main conceptual gap happens when going from two to three layers).  The last section shows, without technical details, how our three-layer results read in the case of deep neural networks with an arbitrary number of layers. Our analysis of linear NN is intended as a step towards understanding the dynamics in the general non-linear case, for which the three questions above are still unresolved.

\subsection{Related work and other limits} 
The first analysis of wide NNs may be traced back to~\cite{neal1996priors, barron1994approximation} for static considerations. The dynamics of wide NNs were first studied in~\cite{jacot2018neural,du2018gradient, daniely2017sgd, allen2019convergence} for the NTP (or related linear dynamics) and in~\cite{nitanda2017stochastic, mei2018mean, chizat2018global, rotskoff2018neural, sirignano2020meanI} for non-linear dynamics in two-layer NNs under $\mu$P, which is known as \emph{mean-field parameterization} in this case (see Remark~\ref{rmk:parameterizations} for a description of these various parameterizations). The importance of the choice of parameterization when passing to the limit was first highlighted in~\cite{chizat2019lazy,mei2019mean} and systematically studied in~\cite{golikov2020towards, yang2021tensor}. Parameterizations akin to $\mu$P were previously empirically studied in~\cite{geiger2020disentangling} as a natural extension of the two-layer mean-field parameterization and a fix to the degeneracy of IP using large initial learning rates was proposed in~\cite{hajjar2021training}.

Our work has strong connections to~\cite{yang2021tensor}, which shows, essentially,  that all the random vectors that are generated when running a finite number of GD steps on a (non-linear) NN converge jointly in law to a family of objects characterized by an abstract algorithm. Because of the intricate dependency that arises between random matrices and random vectors, this limit ``algorithm'' is, unfortunately, more complex than its finite-width counterpart and hard to study beyond a few GD steps. One of our contributions is, for the particular case of linear NNs, to exhibit a simple and theoretically tractable structure in this limit. From a technical viewpoint,~\cite{yang2021tensor} relies on the technique of Gaussian conditioning, which originated in the field of statistical physics to describe TAP equations~\cite{bayati2011dynamics, bolthausen2014iterative}, while we use the \emph{method of moments} which is another classical technique of random matrix theory that allows to easily obtain universality (i.e., our results  apply for non-Gaussian initializations as well; note that  the universality of the technique of~\cite{yang2021tensor} was proved recently in~\cite{golikov2022nongaussian}) and rates that are quantitative in the width. The random matrix statements of our work (in particular Proposition~\ref{prop:recurrence}) have thus their counterpart in the language of~\cite{yang2021tensor}; by proposing an independent proof with different techniques, our purpose is to make the analysis self-contained as well as to shed a different light on the objects appearing in   the limit.

Finally, there is a rich literature on the training dynamics of linear NNs. Some works show that the optimization landscape is benign~\cite{bah2022learning, eftekhari2020training, du2019width} (the latter studies the NTP and thus a dynamics that becomes linear in the large width case), other works study settings where the dynamics display a ``saddle to saddle'' behavior~\cite{li2020towards, jacot2021saddle, gidel2019implicit, saxe2019mathematical} and finally, a line of works studies the implicit bias of gradient descent~\cite{ji2019gradient, arora2019implicit}; that is, which solution is chosen when the problem is underdetermined. Our analysis in the last section borrows ideas from~\cite{ji2019gradient} which shows a  min-$\ell_2$ implicit bias for linear NNs with the logistic loss. Note that linear NN do not always exhibit this type of implicit bias: there are subtle results for architectures that are not fully connected~\cite{gunasekar2018implicit, woodworth2020kernel, pesme2021implicit}.

\subsection{Organization of the paper} In Section~\ref{sec:presentation}, we present our main results and illustrate them with numerical experiments. Section~\ref{sec:basis} studies the structure of iterated products of large random matrices with random vectors, and we show that they can be expanded in a basis of random vectors. These objects are the building blocks of the GD iterations and these results are exploited in Section~\ref{sec:main-proof}, which contains the proof of the infinite-width limit. In Section~\ref{sec:limit-properties} we study properties of the limit system.   Finally, in Section~\ref{sec:multi} we describe the analogous results for multi-layer NNs.

\section{Presentation of the main results}\label{sec:presentation}

 This section presents a rigorous discussion of linear neural networks under $\mu$P of width $m$, in the limit as $m\to \infty$. 
In the case of two-layer neural networks, the analogous problem has been qualitatively understood in \cite{nitanda2017stochastic, mei2018mean, chizat2018global, rotskoff2018neural, sirignano2020meanI, wojtowytsch2020convergence} (see also the reviews \cite{weinenE2020, fernfigalli2022, bach2021gradient}). The first striking special feature of the two layers case is that there is a natural choice of the parametrization --- which mathematically is represented by a suitable factor of $m$ in front of the output weights --- that allows the parameters to remain nondegenerate and deterministic in the limit $m\to \infty$. Under this parametrization, two-layer neural networks can be interpreted as a Wasserstein
gradient flow for the weights (also in the limit), and hence the problem as $m\to \infty$  is also a solution of a Wasserstein gradient flow (and in particular it can be written as  a family of parabolic equations).

For neural networks of more than two layers, several aspects of the previous analysis change. Firstly, as discussed in the introduction, it is not possible to find a natural parametrization (that is, a consistent rescaling of the  three layers of weights) such that one expects them to remain nondegenerate or deterministic in the limit   $m\to \infty$.  In fact, as we will also see a posteriori, with the right choice of parametrization outlined in subsections~\ref{ssec:settings} and \ref{ssec:settings2} below,  the evolution of the entries of the intermediate layer is negligible with respect to their initialization size, but these small variations change significantly the output.
%Hence it is an important choice, discussed in Section ? below, to determine a suitable parametrization. %What are the natural rescalings of our weights in terms of the number of neurons m, so that the limit problem is nondegenerate and each layer of the NN produces a nontrivial contribution to the predictor.
Due to this issue with parametrizations, it is essential in our analysis to consider randomized initial data, and to expect such random effect to survive in our limit system with some averaging effects. %For the same reason, the current analysis is restricted to \emph{linear} neural networks, at difference from the two-layer case.

Our limit system is expressed in a basis of independent, identically distributed gaussian random variables. In turn, its coefficients are obtained by solving an infinitely wide linear neural network, which in the continuous-time limit can be represented as an explicit collection of ODEs.

\subsection{Setting}\label{ssec:settings}
Let $\tilde h^m$ be a single output three-layer linear neural network with input $x\in \R^d$, width $m\in \N$, and weights $\tilde \bU^m\in \R^{m\times d}$, $\tilde \bW^m \in \R^{m\times m}$, and $\tilde \bV^m \in \R^{m}$:
\[
y = \tilde h^m(x, \tilde \bU^m, \tilde \bW^m, \tilde \bV^m) = \sum_{i = 1}^m \tilde \bV^m_i \sum_{j = 1}^m \tilde \bW^m_{ij} \sum_{\ell = 1}^d \tilde \bU^m_{j\ell} x_\ell = \langle \tilde \bV^m , \tilde \bW^m \tilde \bU^m x\rangle.
\]
Given a smooth loss function $\mathcal{L}:\mathbb{R}\times\mathbb{R}\to \R$, we study the behavior of Gradient Descent (GD) starting from a random initialization on the expected loss $\tilde F$ defined as
\[
\tilde F^m(\tilde \bU^m, \tilde \bW^m, \tilde \bV^m) :=\int_{\R^d\times \R}\mathcal{L}(\tilde h^m(x), y)\, d\rho(x, y).
\]
where $\rho \in \mathcal{P}(\R^d\times \R)$ is a probability distribution that represents the input/output data. Specifically, we consider the sequence initialized as
\begin{equation}
\label{eq:init1}
\tilde U^m_{j\ell}(0) \sim \mathcal{N}\left(0, 1\right),\qquad \tilde W^m_{ij}(0) \sim \mathcal{N}\left(0, \frac{1}{m}\right),\qquad \tilde V^m_{i}(0) \sim \mathcal{N}\left(0, \frac{1}{m^2}\right),
\end{equation}
and, with a step-size/learning rate $\tau$, defined recursively as
\[
 \left\{
\begin{split}
\tilde \bU^m(\kappa+1) & = \tilde \bU^m(\kappa) - \tau {\color{red}m} \int \mathcal{L}'( \tilde h^m_{\kappa, \tau}(x), y) \nabla_{\tilde \bU^m} \tilde h_{\kappa, \tau}^m(x)d\rho_\kappa(x, y),\\
\tilde \bW^m(\kappa+1) &= \tilde \bW^m(\kappa) - \tau \int \mathcal{L}'( \tilde h^m_{\kappa, \tau}(x), y)\nabla_{\tilde \bW^m}\tilde h_{\kappa, \tau}^m(x)d\rho_\kappa(x, y),\\
\tilde \bV^m(\kappa+1) &= \tilde \bV^m(\kappa)  - \tau {\color{red}m^{-1}} \int \mathcal{L}'(\tilde h^m_{\kappa, \tau}(x), y)\nabla_{\tilde \bV^m} \tilde h_{\kappa, \tau}^m(x)d\rho_\kappa(x, y).
\end{split}
\right.
\]
where for notational convenience we are denoting 
\[
\begin{split}
\tilde h_{\kappa, \tau}^m(x) & = \tilde h^m(x, \tilde \bU^m(\kappa), \tilde \bW^m(\kappa), \tilde \bV^m(\kappa))\\
 \nabla_{\bullet}\tilde h_{\kappa, \tau}^m(x) & = (\nabla_{\bullet} \tilde h^m)(x, \tilde \bU^m(\kappa), \tilde \bW^m(\kappa), \tilde \bV^m(\kappa)),
\end{split}
\]
and $\mathcal{L}'$ denotes the derivative of the loss function with respect to the first argument. For the sake of generality, we are also considering $\rho_\kappa$ depending on $\kappa$, so that (mini-batch) stochastic gradient descent (SGD) is covered by our analysis. Our only assumption is that these probability measures have uniformly bounded second moments in the first variable:
\begin{equation}
\label{eq:unif_second_moments}
\sup_{\kappa}\int |x|^2\rho_\kappa(x, y) <+\infty. 
\end{equation}
 The factors in red ($m$ and $m^{-1}$) are layer-wise learning rates introduced so that each layer contributes equally to the variations of the predictor in the limit, as the theory will verify.

The randomness of the initialization --- and in particular the large random matrix $\tilde \bW(0)$ --- play a key role in our analysis. The choice of scalings is motivated as follows:
\begin{itemize}
\item The scaling of $\tilde \bU$ and $\tilde \bW$ is chosen so that $\tilde \bU^m(0) x$ and $\tilde \bW^m(0)\tilde \bU^m(0) x$ have a variance that does not depend on $m$ for large $m$ (by the CLT);
\item The scaling of $\tilde \bV$ is of order $1/m$ in order to avoid the lazy training phenomenon~\cite{chizat2019lazy}, that leads to a linear dynamics described in~\cite{jacot2018neural}.
\end{itemize} 
\begin{rem}\label{rmk:parameterizations}
This choice of scale for initialization is referred to as \emph{Maximal Update Parametrization} ($\mu$P) in \cite{yang2021tensor}, where it is shown to lead to feature-learning for each layer\footnote{In our context, there is no feature learning \emph{per se} since the predictor is linear, but we will see that the dynamics remains non-linear in the parameters in the limit (in contrast to NTP).
}. In the introduction, we mentioned NTP, which corresponds to the scales~\eqref{eq:init1} but with $\tilde V_i(0)\sim \mathcal{N}(0,1/m)$; and IP which corresponds to~\eqref{eq:init1} but with $\tilde W_{ij}(0)\sim N(c,1/m^2)$ which is degenerated unless one chooses  $c\neq 0$ or time-dependent learning rates~\cite{hajjar2021training}.
\end{rem}

Computing the gradient using the chain rule, we get the following recursion
\[
 \left\{
\begin{aligned}
\tilde \bU^m(\kappa+1) & =\tilde  \bU^m(\kappa) - \tau {\color{red} m} \tilde \bW^m(\kappa)^\top \tilde \bV^m(\kappa) (\tilde\bxi^m_\kappa)^\top ,\\
\tilde \bW^m(\kappa+1) &= \tilde \bW^m(\kappa) - \tau \tilde \bV^m(\kappa) (\tilde\bxi^m_\kappa)^\top \tilde \bU^m(\kappa)^\top,\\
\tilde \bV^m(\kappa+1) &= \tilde \bV^m(\kappa)  - \tau {\color{red} m^{-1}}\tilde\bW^m(\kappa) \tilde \bU^m(\kappa)\tilde\bxi^m_\kappa.
\end{aligned}
\right.
\]
where we have denoted $\tilde \bxi^m_\kappa := \int \mathcal{L}'(\tilde h^m_{\kappa, \tau}(x), y) x\, d\rho_{\kappa}(x,y)\in \R^d$,

\subsection{Scale-free parameterization}  \label{ssec:settings2} In the theory, it will appear convenient to deal with objects with a scale that is independent of $m$. To this end, we let 
\[
\bZ^m := \sqrt{m}\tilde \bW^m(0)
\]
 (which is a $m\times m$ matrix with independent $\mathcal{N}(0,1)$ entries) and we define:
\begin{equation}
\label{eq:init2}
\begin{split}
\bU^m (\kappa )& := \tilde\bU^m (\kappa ),\\  
 \bW^m (\kappa ) &:= m\big(\tilde \bW^m(\kappa)-\tilde\bW^m (0)\big)=m\tilde\bW^m (\kappa )-\sqrt{m}\bZ^m,\\
\bV^m (\kappa )& :=m \tilde \bV^m (\kappa )
\end{split}
\end{equation}
where the scaling factors are adjusted so that these matrices/vectors have entries of order $1$, as the theory will verify. By definition, $\bU^m(0)$ and $\bV^m(0)$ are random arrays with entries $\mathcal{N}(0, 1)$ and $\bW^m(0)$ is the zero matrix:
\begin{equation}
\label{eq:init11}
 U^m_{j\ell}(0) \sim \mathcal{N}\left(0, 1\right),\qquad W^m_{ij}(0) = 0,\qquad  V^m_{i}(0) \sim \mathcal{N}\left(0, 1\right).
\end{equation}
 The neural network in these new variables becomes
\[
y = h^m(x, \bU^m, \bW^m, \bV^m) = \left\langle \frac{1}{m} \bV^m, \left(\frac{1}{\sqrt{m}}\bZ^m+\frac{1}{m}\bW^m\right) \bU^m x\right\rangle.
\]
The evolution of $(\bU^m(\kappa),\bW^m(\kappa),\bV^m(\kappa))_{\kappa\in \N}$ can be also interpreted as GD (with layer-wise learning rates) on the objective function
\[
F^m(\bU^m,\bW^m,\bV^m) := \int_{\R^d\times \R}\mathcal{L}\left(h^m(x, \bU^m, \bW^m, \bV^m), y\right)\, d\rho(x, y).
\]
We do not explicitly include $\bZ$ in the variables as it is fixed during the training (i.e., we interpret $F^m$ as a random function). All in all, we have
\begin{equation}
\label{eq:training_m}
 \left\{
\begin{split}
\bU^m(\kappa+1) & = \bU^m(\kappa)-  {\tau} \left[\frac{1}{\sqrt{m}}\bZ^m + \frac{1}{m}\bW^m(\kappa)\right]^\top\bV^m(\kappa) (\bxi_{\kappa, \tau}^m)^\top,\\
\bW^m(\kappa+1) &= \bW^m(\kappa) - \tau \bV^m(\kappa)  (\bxi_{\kappa, \tau}^m)^\top (\bU^m(\kappa))^\top,\\
\bV^m(\kappa+1) &= \bV^m(\kappa)  -  \tau \left[\frac{1}{\sqrt{m}}\bZ^m + \frac{1}{m}\bW^m(\kappa)\right]\bU^m(\kappa)\bxi_{\kappa, \tau}^m ,
\end{split}
\right.
\end{equation}
where we have denoted 
$\bxi_{\kappa, \tau}^m =   \int x\, \mathcal{L}'(h^m_{\kappa, \tau}(x), y)   d\rho_\kappa(x, y)\in \R^{d}$ as above, with $ h_{\kappa, \tau}^m(x)  =  h^m(x,  \bU^m(\kappa),  \bW^m(\kappa),  \bV^m(\kappa))$.

%We consider a two-hidden-layers linear neural network with $m$ neurons in each layer, and we consider the evolution of the (discrete-time) gradient descent minimizing the loss to a fixed given measure $\rho(x, y)$, and initialized randomly according to the \emph{Maximal Update Parametrization} from \cite{yang2021tensor}. 

\subsection{Limit dynamics} Our   main result is that, when $m\to\infty$, the training dynamics converge, in a sense detailed below, to some dynamics which are obtained by running the same gradient-based algorithm (i.e., GD or SGD) on an infinitely wide three-layer linear neural network
\begin{align}\label{eq:limit-model-predictor}
\chi (x,\A,\B,\G) = \B^\top \big( \LL +\bG) \A x,
\end{align}
where the variables
\begin{align*}
\A \in\ell^2(\N\times \{1,\dots,d\}) \subset  \R^{\infty\times d}, &&%,&\qquad \text{$\A_{ij}(0) = \delta_{ij}$ for $i\in \N$, $1\le j\le d$}\\
\G \in \ell^2(\N\times \N) \subset \R^{\infty\times\infty},&& %,&\qquad \text{$\G_{i,j}(0) = 0$ for all $i,j\in \N$}\\
\B\in \ell^2(\N) \subset \R^\infty && %,&\qquad \B(0) = (1,0,0,\dots),,
\end{align*}
are initialized with
\begin{equation}
\label{eq:init_inf}
\A(0)  =
\left(
\begin{array}{c}
{\rm Id}_{d} \\
\0_{d\times 1}\\
\0_{d\times 1}\\
\vdots
\end{array}
\right) = 
\left(
\begin{array}{c}
\A_1(0)\\
\vdots\\
\A_d(0)\\
\A_{d+1}(0)\\
\A_{d+2}(0)\\
\vdots
\end{array}
\right)
\in \R^{\infty\times d},\quad
\B(0) = \left(\begin{array}{c}
1 \\
0\\
0\\
\vdots
\end{array}
\right)\in \R^{\infty\times 1},
\end{equation}
where $\A_i(0)\in \R^{1\times d}$ for $i\in \N$ and
\begin{equation}
\label{eq:init_inf2}
(\G)_{ij}(0) = 0\quad\forall(i, j)\in \N^2,
\end{equation}
Also $\LL$ is fixed (not trained) and represents the initialization of the intermediate layer. It is given by 
%\[
% \LL =
% \begin{pmatrix} 
%{\color{red} 0} & {\color{red} \dots} & {\color{red} 0} & 1 & 0 & 0 & \dots\\
%1 & 0 &\dots & 0 & 1 & 0 & \ddots\\
%0 & 1 & 0 & \dots & 0 &1 & \ddots\\
%\vdots & \ddots & \ddots & \ddots & & \ddots & \ddots
%\end{pmatrix}
%\in \R^{\infty\times\infty}
%\]
\begin{equation}
\label{eq:LL_def}
\begin{matrix}
 \LL
 =
 \begin{pmatrix}
 \bovermat{d}{ 0  & \dots & 0 } &
1  & 0  & 0  &  \dots \\

 1 & 0  & 
 \dots  & 0   & 1  & 0  & \ddots \\
 0 & 1 & 0  & \dots  & 0  &1 & \ddots \\
 \vdots & \ddots  & \ddots  & \ddots  &\ddots  & \ddots  & \ddots 
  \end{pmatrix}
 \end{matrix}
 \in \R^{\infty\times\infty},
\end{equation}
 i.e., $\LL = (\Lambda_{ij})_{ij}$ where
\[
 \Lambda_{ij} = 
\left\{
\begin{array}{ll}
1 & \quad\text{if $i+d = j$ or $j+1 = i$},\\
0 & \quad\text{otherwise}.
\end{array}
\right.
\]
The dynamics are therefore given by the following recursion
\begin{equation}
\label{eq:infsystem}
\left\{
\begin{array}{rcl}
\A(\kappa+1)  & = & \A(\kappa) -\tau [\LL + \G(\kappa)]^\top\B(\kappa)\bxi^\top_{\kappa, \tau},\\
{\G}(\kappa+1) & = & \G(\kappa) -\tau \B(\kappa) \bxi^\top_{\kappa, \tau} (\A(\kappa))^\top,\\
{\B}(\kappa+1) & = & \B(\kappa) -\tau [\LL + \G(\kappa)]\A(\kappa)\bxi_{\kappa, \tau},
\end{array}
\right.
\end{equation}
with 
\begin{align*}
\chi_{\kappa, \tau}(x) = \chi(x, \A(\kappa), \G(\kappa), \B(\kappa)) &&\text{and}&&
\bxi_{\kappa, \tau} = \int x\mathcal{L}'(\chi_{\kappa, \tau}(x), y) d\rho_\kappa(x, y)\in \R^d. %,\qquad \chi(x, \A, \G, \B) = \B^\top[\LL + \G]\A x.
\end{align*}
When $\rho_\kappa =\rho$ for all $\kappa\in \N$, this recursion is exactly the GD on the (deterministic) objective function $\mathcal{E}$ defined by
\begin{align}\label{eq:limit-objective}
\mathcal{E}(\A,\G,\B) = \int \mathcal{L}(\B^\top (\LL+\bG)\A x, y)d\rho(x,y).
\end{align}
%The system \eqref{eq:infsystem} has elements $\A\in \R^{\infty\times d}, \G\in \R^{\infty\times\infty}, \B\in \R^{\infty}$ initialized as 

\subsection{Main statements} \label{ssec:main-statements} Let us consider two families of independent infinite Gaussian vectors
\begin{equation}
\label{eq:basis_gamma}
(\bGamma_1,\bGamma_2,\dots)\qquad \text{and}\qquad (\tilde \bGamma_1,\tilde \bGamma_2,\dots),
\end{equation}
where the entries of $\bGamma_k,\tilde \bGamma_k \in \R^\N$ are all independent $\mathcal{N}(0,1)$ random vectors. We define 
\begin{equation}
\label{eq:lim_evo}
\left\{
\begin{split}
\bU^\infty(\kappa) & = \sum_{i \ge 1} \bGamma_i\A_i(\kappa),\\
\bW^\infty(\kappa) & =  \sum_{i, j \ge 1}  \tilde{\bGamma}_i{\bGamma_j}^\top G_{ij}(\kappa),\\
\bV^\infty(\kappa) & =  \sum_{i \ge 1}\tilde\bGamma_i B_i(\kappa) .
\end{split}
\right.
\end{equation}

%where $\bGamma_k, \tilde\bGamma_k \sim \mathcal{N}(\0_{\infty\times 1}, {\rm Id}_\infty)$ for $k \in \N$ are infinite arrays, and all together form a matrix normal distribution, $\mathbb{G}, \tilde{\mathbb{G}}   \sim \mathcal{M}\mathcal{N}_{\infty\times\infty}(\0_{\infty\times\infty}, {\rm Id}_{\infty\times\infty}, {\rm Id}_{\infty\times\infty})$ and $(\mathbb{G} , \tilde{\mathbb{G}})\sim \mathcal{M}\mathcal{N}_{\infty\times\infty}(\0_{\infty\times\infty}, {\rm Id}_{\infty\times\infty}, {\rm Id}_{\infty\times\infty})$ as well. 

We shall prove the convergence in distribution, as $m \to \infty$, of the finite dimensional time-discretized dynamics to the infinite one (see Definition~\ref{defi:inf_dist_conv} for the precise definition of convergence that we use). The proof of convergence will rely on the method of moments: we will prove that the moments of our random variables converge to the ones of the limit as $m\to \infty$, and this implies convergence in distribution. Our main theorem is the following:

\begin{thm}[Infinite-width limit]
\label{thm:main}
Let $\tau > 0$ be fixed, let $\mathcal{L}$ be such that $\mathcal{L}''$ is bounded, and let us suppose that \eqref{eq:unif_second_moments} holds. 

Let $(\bU^m(\kappa), \bW^m(\kappa), \bV^m(\kappa))_{\kappa \in \N}$ be the solution to \eqref{eq:training_m} with initialization \eqref{eq:init11}, and let $(\bU^\infty(\kappa), \bW^\infty(\kappa), \bV^\infty(\kappa))_{\kappa \in \N}$ be given by \eqref{eq:lim_evo} (see \eqref{eq:init_inf}-\eqref{eq:init_inf2}-\eqref{eq:infsystem}).  Then, for any stopping time $\kappa_* \in \N$, 
\[
\begin{array}{c}
\big( (\bU^m(0), \bW^m(0), \bV^m(0)), \dots, (\bU^m(\kappa_*), \bW^m(\kappa_*), \bV^m(\kappa_*))\big)\\
\downarrow {\rm d.}\\
\big((\bU^\infty(0), \bW^\infty(0), \bV^\infty(0)), \dots, (\bU^\infty(\kappa_*), \bW^\infty(\kappa_*), \bV^\infty(\kappa_*))\big)
\end{array}
\]
as $m\to \infty$. Moreover, considering the vectors in $\R^d$ that represent the linear predictors of the neural network
\begin{align}
\lambda^m(\kappa) &= \bU^m(\kappa)^\top (m^{-1/2}\bZ^m +m^{-1}\bW^m(\kappa))^\top (m^{-1} \bV^m(\kappa))\\
\lambda^\infty(\kappa) &= \A(\kappa)^\top (\LL+\G(\kappa))^\top \B(\kappa)\label{eq:infinite-width-linear}
\end{align}
 we have  $\lambda^m(\kappa)\overset{a.s.}{\to}\lambda^\infty(\kappa)$ for every $\kappa\in \N$ quantitatively (see \eqref{eq:quant_convergence} below).
\end{thm}

We can make the following remarks :
\begin{enumerate}[(i)]
%\item The limit dynamics is deterministic in the space of predictors (described by $\lambda^\infty$) and random in the space of parameters (described by $(\A,\G,\B)$). 
\item Since $(\bU^\infty_j(\kappa),\bW^\infty_{i,j}(\kappa),\bV^\infty_i(\kappa))_{\kappa =1}^{\kappa^*}$  is a \emph{separately exchangeable} $\mathbb{R}^{3\kappa^*}$-valued random array, the dependency structure between its entries that we obtain in Theorem~\ref{thm:main} is consistent, as it should, with the Aldous-Hover representation of infinite exchangeable arrays~\cite[Thm.~1.4]{aldous1981representations}, which is a generalization of De Finetti's theorem. See~\cite{oh2021gradient} for a study of gradient flows with a similar dependency structure.
%{\color{red}\item This theorem shows, in particular, \emph{propagation of chaos} for this dynamics. Indeed, looking at \eqref{eq:lim_evo} for some $\kappa \in \mathbb{N}$ fixed (and $d=1$ for simplicity), we see that the law of $(\bU^\infty_j(\kappa),\bW^\infty_{i,j}(\kappa),\bV^\infty_i(\kappa))$ is a  \emph{separately exchangeable} $\mathbb{R}^3$-valued random array. This notion generalizes the notion of exchangeable random sequences, which is obtained in standard (single-index) propagation of chaos results (see~\cite{oh2021gradient} for a study of certain dynamical systems with a similar dependency structure).}
\item A perhaps counter-intuitive consequence of this theorem is that, even if this parametrization $\mu$P preserves \emph{feature-learning} in the limit, the evolution of the entries of the intermediate layer $\tilde \bW^m_{i,j}(\kappa)-\tilde \bW^m_{i,j}(0)$ (of order $1/m$) is negligible in front of their magnitude at initialization $\tilde \bW^m_{i,j}(0)$ (of order $1/\sqrt{m}$). Still, these small variations collectively create a significant variation of the output.
\item In the proof of this theorem, the convergence of the predictor is quantified as 
\begin{equation}
\label{eq:quant_convergence}
\E\left[\|\lambda^m(\kappa) - \lambda^\infty(\kappa)\|^2\right]\le C_{\eps,\kappa} m^{-1+\eps},
\end{equation}
for any $\eps > 0$ and for some $C_{\eps,\kappa}$ depending on $\eps > 0$ and $\kappa$, but independent of $m$. As can be seen from numerical experiments (see Figure~\ref{fig:width-convergence}-(C)) this convergence is expected to be (almost) optimal, which is also consistent with the fact that it comes from a Central Limit Theorem. 
\item Our proofs are based on universality properties and only use that $\bZ^m$ has i.i.d. subgaussian entries with zero mean and unit variance. In particular, the previous statement is also true for these more  general initializations of the $\tilde \bW^m$ weights. 
\item \label{it:rmk} If we want to take more general subgaussian initializations  $ \bU^m(0)$ and $ \bV^m(0)$ we can also do it, provided that in the previous statement (more precisely, in \eqref{eq:basis_gamma}) we change $\bGamma_1$ and $\tilde \bGamma_1$ by $ \bU^\infty(0)$ and $ \bV^\infty(0)$; see Figure~\ref{fig:non-Gaussian}.
\end{enumerate}

Our second statement studies the behavior of the limit model, which is an infinitely wide linear neural network with a particular \emph{deterministic} initialization. For the sake of simplicity, we consider the continuous-time limit $\tau\to 0$ of the dynamics, that is, the gradient flow of the functional $F^\infty$ and the corresponding linear predictor $(\lambda^\infty(t))_{t\geq 0}$.
\begin{thm}
\label{thm:main2}
Consider the square loss $\mathcal{L}(\hat y,y)=\frac12 \vert y-\hat y\vert^2$ and assume that $\rho$ has finite second moments. Then $\lambda^\infty(t)$ converges at an exponential rate to the minimal $\ell_2$-norm minimizer of the risk $\lambda \mapsto \frac12 \int \vert \lambda^\top x-y\vert^2 d\rho(x,y)$.
\end{thm}
Note that this \emph{implicit bias} towards min-$\ell_2$ norm solutions is not a particularly impressive property as such, since just the basic gradient flow on the square-loss initialized from $0$, i.e.,
\begin{align}\label{eq:linear-dynamics}
\lambda_{\mathrm{gf}}(0)=0&& \lambda'_{\mathrm{gf}}(t) = -\int x(\lambda_{\mathrm{gf}}(t)^\top x-y)d\rho(x,y),
\end{align}
satisfies the same statement (notice, however, that our dynamics are truly non-linear, see Figure~\ref{fig:time-convergence}). This result is mostly intended to highlight the fact that our characterization of the infinite-width dynamics is precise enough to obtain such properties.

\subsection{Numerical illustrations}
We consider GD for the finite-width and infinite-width models, with input dimension $d=10$, the square loss, a data distribution given by $x\sim \mathcal{N}(0,{\rm Id}_{d})$ and $y=x^\top \lambda^*$ for some $\lambda^*\in \R^d$ that is randomly drawn from $\mathcal{N}(0,{\rm Id}_{d})$. The code to reproduce the experiments is available online\footnote{\url{https://github.com/lchizat/2022-wide-linear-NN}}.

Figure~\ref{fig:width-convergence}  illustrates the convergence to the limit model as the width $m\to \infty$, with a step-size $\tau=0.2$. In (A), we show the path of $(\lambda^m(\kappa))_{\kappa\geq 0}$ projected on two first coordinates of $\mathbb{R}^d$. We observe that, as the width increases, they follow a trajectory approaching that of the limit $(\lambda^\infty(\kappa))_{\kappa\geq 0}$, which starts at $\lambda^\infty(0)=0$ and converges to the min-$\ell_2$ norm predictor $\lambda^*$ shown as a red diamond, and computed via the pseudo-inverse formula. In (B) we represent the rate of convergence in $m$ of the predictor as a function of the width, at both initialisation and large time. As it can be seen, it corresponds to \eqref{eq:quant_convergence} with $\eps = 0$. Finally, in (C), we represent the mean square of the entries of $\bV^m(\kappa)$, computed as $v_\kappa = \frac1m \sum_{i=1}^m\bV_j(\kappa)^2$ (which is also a proxy for the variance of $\bV^m_j(\kappa)$ for $1\leq j\leq m$ since the entries of $\bV^m(\kappa)$ are asymptotically independent) and its limit which is $\Vert \B(\kappa)\Vert_2^2$ by~\eqref{eq:lim_evo}. This is just a simple example of a statistics described by our limit model.

\begin{figure}
\centering
\begin{subfigure}{0.49\linewidth}
\centering
\includegraphics[scale=0.45]{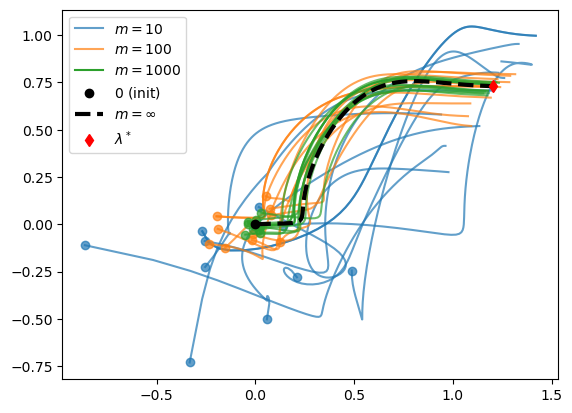}
\caption{Convergence in predictor space}
\end{subfigure}
\begin{subfigure}{0.49\linewidth}
\centering
\includegraphics[scale=0.5]{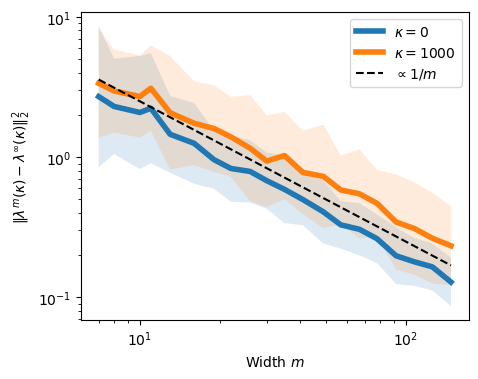}
\caption{Rate of convergence of the predictor}
\end{subfigure}
\begin{subfigure}{0.49\linewidth}
\centering
\includegraphics[scale=0.5]{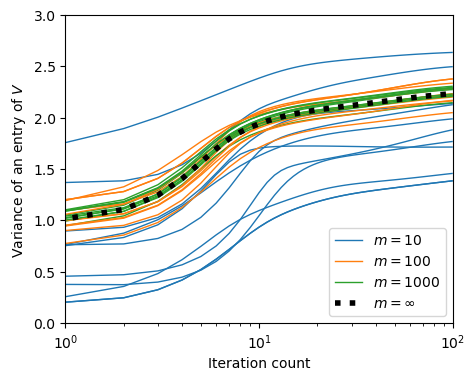}
\caption{Convergence in parameter space}
\end{subfigure}
\caption{Convergence to the limit model (A) Trajectory of the predictor $\lambda^m(t)$, projection on the two first coordinates (bullets represent $\lambda^m(0)$). (B) Rate of convergence of the predictor as a function of the width $m$, at initialization $\kappa=0$ and large time $\kappa=1000$ (shaded area represent standard deviation over $50$ repetitions). (C) Evolution of the average square of an entry of $\bV^m$ and in the limit.}\label{fig:width-convergence}
\end{figure}

In Figure~\ref{fig:time-convergence}, we 
 take a small step-size to approximate the gradient flow $\tau=0.001$ and explore the behavior of the limit model. It can be seen from the GD equations that at   $\kappa$ steps, only the first $d\cdot \kappa$ rows of $\A(\kappa)$ and of $\B(\kappa)$ are non-zero. Thus the infinite model can be trained exactly for a bounded number of steps\footnote{We also noticed that truncating the limit model~\eqref{eq:limit-model-predictor} below this size introduces an error that decays exponentially in the width, instead of the $m^{-1/2}$ rate for the randomly initialized model.}. We also introduce a fixed scale parameter $s>0$ that multiplies the predictor, which is equivalent to scaling the standard deviation of the initialization by $s^{1/3}$ at each layer. By~\cite{chizat2019lazy} and since $\lambda^\infty(0)=0$, we know that as $s\to \infty$, the dynamics converges to the linear dynamics~\eqref{eq:linear-dynamics}. This illustration confirms that the dynamics of $\lambda^\infty$ is non-linear (unless $s\to \infty$), although it has the same endpoints at $t=0$ and $t=\infty$ as the linear dynamics. For small scales, $s\ll 1$, we observe on the right plot that the objective function starts with a plateau; this is reflected by our convergence analysis in Proposition~\ref{prop:expconvergence}, which is a two-phase analysis: a first phase to escape from the initialization (which is close to a stationary point when $s\ll 1$) and a second phase with exponential convergence. We note that the convergence speeds in this plot are not directly comparable because we did not attempt to find the best step-size $\tau$ for various values of $s$.
 
 \begin{figure}
\centering
\begin{subfigure}{0.49\linewidth}
\centering
\includegraphics[scale=0.55]{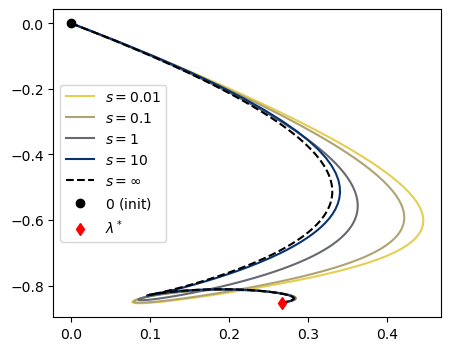}
\caption{Trajectory of GD (projection)}
\end{subfigure}
\begin{subfigure}{0.49\linewidth}
\centering
\includegraphics[scale=0.5]{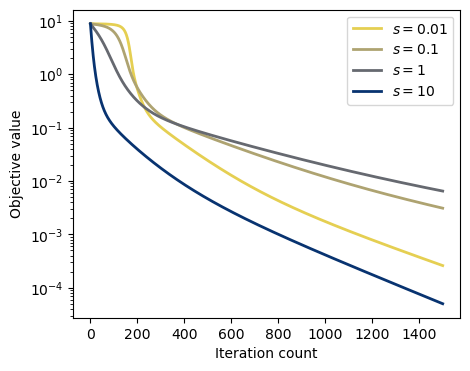}
\caption{Evolution of the objective function}
\end{subfigure}
\caption{Behavior of the limit model and effect of the scale parameter $s$. (left) Projection of $\lambda^m(t)$ on the two first coordinates. (right) Evolution of the loss.}\label{fig:time-convergence}
\end{figure}
 
 Finally, in Figure~\ref{fig:non-Gaussian} (A) we plot the distribution of the weights at large times with non-Gaussian initialization (in blue). As discussed above (remark~\eqref{it:rmk}) the weights are never Gaussian in this case (not even in the large time limit) since in general the first coefficient in the basis (e.g. $\B_1(\kappa)$) does not necessarily vanish at $\kappa = \infty$. However, the analysis described in Theorem~\ref{thm:main} still works and the non-gaussianity of the weights is only due to the interference of this first element of the basis: the other elements are still Gaussian. In the figure, this can be seen by subtracting the first element from the distribution of weights, where we recover a Gaussian profile (in orange). In any case, we still expect a rate of convergence to the minimizer given by the rate of the Central Limit Theorem (Figure~\ref{fig:non-Gaussian}~(B)).

\begin{figure}
\centering
\begin{subfigure}{0.49\linewidth}
\centering
\includegraphics[scale=0.55]{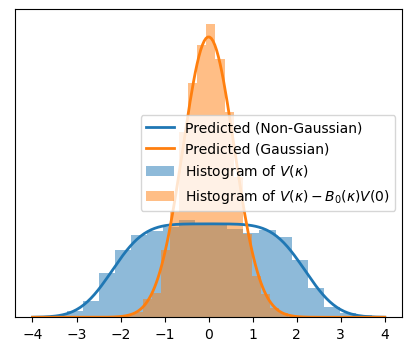}
\caption{Distribution of $\bV(\kappa=20)$.}
\end{subfigure}
\begin{subfigure}{0.49\linewidth}
\centering
\includegraphics[scale=0.5]{./images/3LNN-width-speed}
\caption{Rate of convergence of the predictor}
\end{subfigure}
\caption{Illustration for a non-Gaussian initialization (centered uniform distribution with the same variance as in the Gaussian case). (A) The distribution of parameters is non-Gaussian at all times, but its exact shape can be computed using the limit model (see remark~\eqref{it:rmk} after Theorem~\ref{thm:main}) (B) The convergence of the predictor to the limit model happens at the same rate as in the Gaussian case.}\label{fig:non-Gaussian}
\end{figure}

\section{An independent family of Gaussian vectors}\label{sec:basis}

\subsection{Notation}
\label{ssec:notation}
We denote vectors and matrices with bolded symbols (except for $x\in \R^d$), and scalars with plain symbols. Given an element $\M\in \R^{m\times n}$ with $m, n\in \N\cup\{\infty\}$, we denote 
\[
\|\M\|^2 := \sum_{i=1}^m\sum_{j = 1}^n |M_{ij}|^2.
\]
More generally, for any $p\ge 1$, we denote
\[
\|\M\|^p_p := \sum_{i=1}^m\sum_{j = 1}^n |M_{ij}|^p.
\]

When $\M$ is a random array in a probability space $(\Omega, \mathcal{F}, \PP)$, we still denote $\M\in \R^{m\times n}$ where now it is implicitly evaluated at an element of the sample space $\omega\in \Omega$. In particular, $\|\M\|^2 = \|\M(\omega)\|^2$ where the evaluation will be implicit whenever there is no ambiguity. It must not be confused with $\E[\|\M\|^2]$, which is given by 
\[
\E[\|\M\|^2] = \int_\Omega \|\M(\omega)\|^2 d\PP(\omega). 
\] 

Finally, let us give the following definition on the convergence in law/distribution for arrays of increasing size: 
\begin{defi}\label{defi:inf_dist_conv}
Given a family of random vectors $(\bX^m_i)_{i\in \N}$ with $\bX_i^m\in \R^m$, we say that they converge in distribution to a family of infinite-dimensional random vectors $(\bX^\infty_i)_{i\in \N}$ with $\bX_i^\infty\in \R^\infty$ if for every fixed $M, N\in \N$, the family $((\bX^m_i)_{1..N})_{1\le i\le M}$ converges in distribution to $((\bX^\infty_i)_{1..N})_{1\le i\le M}$ (where we have denoted, for $\bX\in \R^m$, $\bX_{1..N}$ its first $N$ components; which is always well defined, for $m$ large enough). 
\end{defi}

\subsection{The objects} 
\label{sec:theobjects}
For the sake of readability, we first construct the independent family that will act as a basis of our evolution in the unidimensional input case. We refer to Section~\ref{sec:multid} below for the statements in the multi-dimensional input case, where the proofs are essentially the same. 

Let $\bU$ and $\bV$ be two infinite random vectors,
\begin{equation}
\label{eq:JKdef0}
\bU := \begin{pmatrix}
    U_1        \\
    U_2        \\
    \vdots       
\end{pmatrix},\qquad \bV := \begin{pmatrix}
    V_1        \\
    V_2        \\
    \vdots       
\end{pmatrix},
\end{equation}
with entries $(U_i)_{i\in \N}$ and $(V_i)_{i\in \N}$ that are independent random variables with $U_i, V_i\sim \mathcal{N}(0,1)$ for $i\in \N$.
 
Let $\bZ$ be an infinite random matrix,
\[
\bZ := \begin{pmatrix}
    Z_{11}       & Z_{12} & \dots   \\
    Z_{21}       & Z_{22} & \dots \\
    \vdots       & \vdots & \ddots 
\end{pmatrix},
\]
 whose entries $(Z_{ij})_{i, j \in\N}$ are independent random variables $Z_{ij}\sim \mathcal{N}(0,1)$ for all $i,j\in \N$, and also independent from $(U_i)_{i\in \N}$ and $(V_i)_{i\in \N}$.
 
Let us denote by $\bU^m$ the restriction of $\bU$ to the first $m$  entries, $\bU^m \in \R^m$, with $U^m_i = U_i$ for $1\le i \le m$ (respectively $\bV^m$). Similarly, we denote by $\bZ^m$ the restriction of $\bZ$ to the first  $m\times m$ entries,  $\bZ^m\in \R^{m\times m}$,  with $Z^m_{ij} = Z_{ij}$ for $1\le i,j\le m$. 

Let us denote by $\bXi^m_{i}(k)$ the set of bipartite (directed) $k$-chains between two equal sets of indices $I_0 = I_1 = \{1,\dots,m\}$ that start in $I_1$ and end at $i$, where $i\in I_0$ if $k$ is odd, and $i\in I_1$ if $k$ is even. Namely, 
\[
\begin{split}
\bXi^m_{i}(k):= \{& \left((i_2,i_1), (i_2, i_3), (i_4, i_3),(i_4,i_5),  \dots, (i_{k-1}, i_k), (i, i_k)\right):\\
& ~~ i, i_{2j}\in I_0~~\text{for $1\le j\le (k-1)/2$},~~i_{2j-1}\in I_1~~\text{for $1\le j\le (k+1)/2$} \}
\end{split}
\]
if $k\in \N$ is odd, and 
\[
\begin{split}
\bXi^m_{i}(k):= \{& \left((i_2,i_1), (i_2, i_3), (i_4, i_3),(i_4,i_5),  \dots, (i_k, i_{k-1}), (i_k, i)\right):\\
& ~~ i_{2j}\in I_0~~\text{for $1\le j\le k/2$},~~i, i_{2j-1}\in I_1~~\text{for $1\le j\le k/2$} \}
\end{split}
\]
if $k\in \N$ is even. In particular, an element $\Xi\in \bXi^m_i(k)$ is of the form $\Xi = (\Xi_1,\dots,\Xi_k)$ where $\Xi_\ell = ((\Xi_\ell)_1, (\Xi_\ell)_2)$ with $(\Xi_\ell)_1\in I_0$ and $(\Xi_\ell)_2\in I_1$, for $ 1\le \ell \le k$, 
\[
\Xi_\ell = (i_\ell, i_{\ell+1})\quad\text{if $\ell$ is even}, \qquad \Xi_\ell = (i_{\ell+1}, i_{\ell})\quad\text{if $\ell$ is odd}, \qquad \text{and $i_{k+1} = i$}. 
\]

We  think of each  $\Xi_\ell$ for $1\le \ell\le k$ as possible indices of a matrix $m\times m$. In this way, if $\Xi\in \bXi^m_i(k)$  we denote  $\Xi^\top := (\Xi^\top_1,\dots, \Xi^\top_K)$ where $\Xi^\top_\ell := ((\Xi_\ell)_2,(\Xi_\ell)_1)$ (that is, transposing every matrix; or alternatively, reflecting the bipartite chain).

We can use the previous definitions to compute iterative multiplications of $(\bZ^m)^\top$ and $\bZ^m$ against $\bU^m$. That is, (using $(\Xi_1)_2 = i_1$ in the notation above)
\begin{equation}
\label{eq:mult_odd}
\begin{split}
\left(\bZ^m[(\bZ^m)^\top(\bZ^m)]^{\frac{k-1}{2}}\bU^m\right)_i & = \big(\overbrace{\bZ^m (\bZ^m)^\top\dots \bZ^m}^{k} \bU^m\big)_i\\
& \hspace{-2cm} = \sum_{i_1,\dots,i_k = 1}^m Z_{i, i_k}\dots Z_{i_2,i_3}Z_{i_2,i_1} U_{i_1}= \sum_{\Xi \in \bXi^m_{i}(k)} \left(\prod_{\ell = 1}^k Z_{\Xi_\ell}\right) U_{(\Xi_1)_2}
\end{split}
\end{equation}
if $k\in \N$ is odd, and 
\begin{equation}
\label{eq:mult_even}
\begin{split}
\left([(\bZ^m)^\top(\bZ^m)]^{\frac{k}{2}}\bU^m\right)_i & = \big(\overbrace{(\bZ^m)^\top \bZ^m\dots \bZ^m}^{k} \bU^m\big)_i\\
&  \hspace{-2cm}= \sum_{i_1,\dots,i_k = 1}^m Z_{i_k, i}\dots Z_{i_2,i_3}Z_{i_2,i_1} U_{i_1}= \sum_{\Xi \in \bXi^m_{i}(k)} \left(\prod_{\ell = 1}^k Z_{\Xi_\ell}\right) U_{(\Xi_1)_2}
\end{split}
\end{equation}
if $k\in \N$ is even. 

Let us denote by $\# v(\Xi)$ with $\Xi\in \bXi^m_i(k)$ the number of vertices seen by the $k$-chain $\Xi$, namely,\footnote{Here and in the sequel, given a finite set $A$, we denote by $|A|$ its cardinality.}
\[
\begin{split}
\# v(\Xi) &= |\left\{i_2,i_4,\dots, i_{k-1}, i\right\}|+|\left\{i_1,i_3,\dots, i_{k}\right\}|,\\
&\text{ where } \Xi = \left((i_2,i_1), (i_2, i_3),  \dots, (i_{k-1}, i_k), (i, i_k)\right)\in \bXi^m_i(k)
\end{split}
\]
if $k\in \N$ is odd, and 
\[
\begin{split}
\# v(\Xi) &= |\left\{i_2,i_4,\dots, i_{k}\right\}|+|\left\{i_1,i_3,\dots, i_{k-1}, i\right\}|,\\
&\text{ where } \Xi = \left((i_2,i_1), (i_2, i_3),  \dots, (i_{k}, i_{k-1}), (i_k, i)\right)\in \bXi^m_i(k)
\end{split}
\]
if $k\in \N$ is even. 

Let us define ${\tilde\bXi}^m_i(k)$ to be the subset of $\bXi^m_i(k)$ with chains that contain no loops (alternatively, chains that  visit  each vertex at most once), 
\[
{\tilde\bXi}^m_i(k):= \left\{\Xi \in \bXi^m_i(k) : \# v(\Xi) = k+1\right\}\subset \bXi^m_i(k).
\]

Finally, we define (cf. \eqref{eq:mult_odd}-\eqref{eq:mult_even}),
\begin{equation}
\label{eq:JKdef2}
\begin{split}
J^m_{k,i} & := \frac{1}{m^{k/2}}\sum_{\Xi \in {\tilde \bXi}^m_{i}(k)} \left(\prod_{\ell = 1}^k Z_{\Xi_\ell}\right)U_{(\Xi_1)_2},\\
K^m_{k,i} & = \frac{1}{m^{k/2}}\sum_{\Xi \in {\tilde \bXi}^m_{i}(k)} \left(\prod_{\ell = 1}^k Z_{\Xi^\top_\ell}\right)V_{(\Xi_1)_2},
\end{split}
\end{equation}
namely, we consider the product in \eqref{eq:mult_odd}-\eqref{eq:mult_even} (both against $\bU$ and $\bV$) but we keep only those elements of the sum that have no loops (and we rescale by the appropriate size, where the size-preserving objects are $m^{-1/2} \bZ^m$). Notice that in the multiplication against $\bV$ we are considering matrices $\bZ^\top$ instead of $\bZ$. In particular, we could alternatively think of $K_{k, i}^m$ as the (renormalized) sum over loopless chains starting from the set $I_0$ and ending at $i$ and with length $k$, where now $i\in I_0$ if $k$ is even and $i\in I_1$ if $k$ is odd (the opposite from before).

We define the vectors
\begin{equation}
\label{eq:JKdef}
\bJ^m_{k} := \begin{pmatrix}
    J^m_{k,1}        \\
    J^m_{k,2}        \\
    \vdots     \\
    J^m_{k,m}  
\end{pmatrix}\quad\text{and}\quad \bK^m_{k} := \begin{pmatrix}
    K^m_{k,1}        \\
    K^m_{k,2}        \\
    \vdots     \\
    K^m_{k,m}  
\end{pmatrix},\quad\text{with $J^m_{k,i}$ and $K^m_{k,i} $ given by \eqref{eq:JKdef2}}.
\end{equation}
(We denote $\bJ_0^m = \bU^m$ and $\bK_0^m = \bV^m$.) Let us also define, for $k\in \N$, $\{\bJ_k\}_{k\in \N}$ and $\{\bK_k\}_{k\in \N}$ families of independent, identically distributed (infinite) random vectors
\begin{equation}
\label{eq:JKdef_lim}
\bJ_{k} := \begin{pmatrix}
    J_{k,1}        \\
    J_{k,2}        \\
    \vdots    
\end{pmatrix}\qquad\text{and}\qquad \bK_{k} := \begin{pmatrix}
    K_{k,1}        \\
    K_{k,2}        \\
    \vdots      
\end{pmatrix}.
\end{equation}
with $J_{k, i}, K_{k, i}\sim \mathcal{N}(0,1)$ and all independent between them. 

\subsection{Convergence of the family} Let us now prove that the family of vectors $\{(\bJ_k^m, \bK_k^m)\}_{k\in \N}$  converges in distribution to the family of i.i.d. Gaussian vectors. We will in fact prove convergence of all the moments. 

\begin{thm}
\label{thm:main_convergence}
The family of vectors $\{(\bJ_k^m, \bK_k^m)\}_{k\in \N}$ defined in \eqref{eq:JKdef} converges, in distribution, to the family $\{(\bJ_k, \bK_k)\}_{k\in \N}$ (according to Definition~\ref{defi:inf_dist_conv}).
\end{thm}

This statement says that products of the form \eqref{eq:mult_odd}-\eqref{eq:mult_even} have a simple asymptotic structure \emph{provided that we remove all the chains of indices with loops.} The chains with loops add correlations, which are described in the next proposition on the recursion property for this family. 

\begin{proof}
We  will show that, for each $N_J, N_K \in \N$ and every fixed families of pairs of indices $(k_1, i_1),\dots, (k_{N_J}, i_{N_J})$ and $(\ell_1, j_1),\dots, (\ell_{N_K}, j_{N_K})$, we have
\[
(J_{k_1,i_1}^m,\dots J^m_{k_{N_J}, i_{{N_J}}},K_{\ell_1,j_1}^m,\dots K^m_{\ell_{N_K}, j_{{N_K}}}) \xrightarrow[]{d.} (J_{k_1,i_1},\dots J_{k_{N_J}, i_{{N_J}}},K_{\ell_1,j_1},\dots K_{\ell_{N_K}, j_{{N_K}}})
\]
as $m\to \infty$. 

We use the method of moments. Let us fix the indices $(k_1, i_1),\dots, (k_{N_J}, i_{N_J})$ and $(\ell_1, j_1),\dots, (\ell_{N_K}, j_{N_K})$, and the powers $p_1,\dots,p_{N_J},q_1,\dots,q_{N_K}\in \N$, and let 
\begin{equation}
\label{eq:Em}
\mathcal{E}_m := \E\left[ (J_{k_1,i_1}^m)^{p_1}\dots (J^m_{k_{N_J}, i_{{N_J}}})^{p_{N_J}}(K_{\ell_1,j_1}^m)^{q_1}\dots (K^m_{\ell_{N_K}, j_{{N_K}}})^{q_{N_K}}\right]
\end{equation}
where we assume that $J_{k_1,i_1}^m, \dots, J_{k_{N_J},i_{N_J}}^m, K_{\ell_1,j_1}^m, \dots, K_{\ell_{N_K},j_{N_K}}^m$  are all different. We will show that
\[
\mathcal{E}_m \xrightarrow[]{m\to\infty} \mu_{p_1}\dots \mu_{p_{N_J}}\mu_{q_1}\dots\mu_{q_{N_K}},
\]
where $\mu_k$ denotes the $k$-th plain moments of a normal distribution $\mathcal{N}(0,1)$,
\[
\mu_k = \left\{
\begin{array}{ll}
0& \quad \text{if $k\in \N$ is odd},\\
(k-1)!! & \quad \text{if $k\in \N$ is even},
\end{array}
\right.
\]
with $(k-1)!! = (k-1)\cdot (k-3)\dots\cdot 5\cdot3\cdot 1$ being the  double factorial. This will directly give the desired result.

Recall that each element $J_{k_1,i_1}^m, \dots, J_{k_{N_J},i_{N_J}}^m, K_{\ell_1,j_1}^m, \dots, K_{\ell_{N_K},j_{N_K}}^m$ can be thought of as a sum over bipartite loopless chains between $I_0$ and $I_1$, starting at $I_1$ for $J_{k_\alpha, i_\alpha}^m$, starting at $I_0$ for $K_{k_{\alpha'}, i_{\alpha'}}^m$, and ending at $i_1,\dots, i_{N_J}, j_1, \dots, j_{N_k}$ with length $k_1,\dots,k_{N_J},\ell_1,\dots,\ell_{N_K}$, respectively; also, each ending vertex belongs to either $I_0$ or $I_1$ depending on the parity of the length ($k_\alpha$ or $\ell_{\alpha'}$). We denote by $\# v_{\rm end}$ the number of different  such ending vertices,
\[
\# v_{\rm end} := \left|\bigcup_{\substack{\alpha = 1\\k_\alpha \equiv 1~\text{mod}~2}}^{N_J} \{i_\alpha\}\cup\bigcup_{\substack{\alpha' = 1\\\ell_{\alpha'} \equiv 0~\text{mod}~2}}^{N_K} \{j_{\alpha'}\}\right|+\left|\bigcup_{\substack{\alpha = 1\\k_\alpha \equiv 0~\text{mod}~2}}^{N_J} \{i_\alpha\}\cup\bigcup_{\substack{\alpha' = 1\\\ell_{\alpha'} \equiv 1~\text{mod}~2}}^{N_K} \{j_{\alpha'}\}\right|.
\]

Let us define by $\mathfrak{G}$ the family of bipartite graphs (each graph is seen as a disjoint union of edges) connecting the sets of vertices $I_0$ and $I_1$ appearing in the expansion of the definition of $\mathcal{E}_m$. Namely, any graph $G\in\mathfrak{G}$ is a set of edges between $I_0$ and $I_1$ given by the (disjoint) union of $p_1$ elements in $\tilde \bXi_{i_1}^m(k_1)$, $p_2$ elements in $\tilde \bXi_{i_2}^m(k_2)$, $\dots$, and $q_{N_K}$ elements in $\tilde \bXi_{j_{N_K}}^m(\ell_{N_K})$:
\begin{equation}
\label{eq:G}
\begin{split}
\mathfrak{G} :=  \bigg\{  G & :  G = \bigsqcup_{\alpha = 1}^{N_J}\bigsqcup_{\beta = 1}^{p_\alpha} \Xi^{\beta,\alpha} \sqcup \bigsqcup_{\alpha' = 1}^{N_K}\bigsqcup_{\beta' = 1}^{q_{\alpha'}} \Theta^{\beta',\alpha'},
 \quad \text{for some}\quad \Xi^{\beta,\alpha} \in \tilde \bXi^m_{i_\alpha}(k_\alpha)\\
 & 
 \quad \text{and}\quad \Theta^{\beta',\alpha'} \in \tilde \bXi^m_{j_{\alpha'}}(\ell_{\alpha'})\quad\text{with}\quad 1 \le \alpha \le N_J, 1\le \alpha' \le N_K\bigg\}.
\end{split}
\end{equation}
Observe that each element $G\in \mathfrak{G}$ contains $\#e(G)$ edges (with multiplicity), where 
\begin{equation}
\label{eq:defN}
\# e(G) = k_1p_1+\dots +k_{N_J}p_{N_J}+\ell_1q_1+\dots +\ell_{N_K}q_{N_K} =: N,
\end{equation}
which is independent of the element $G\in \mathfrak{G}$ chosen. 

Also, given a fixed element 
\[
 \mathfrak{G}\ni G = \bigsqcup_{\alpha = 1}^{N_J}\bigsqcup_{\beta = 1}^{p_\alpha} \Xi^{\beta,\alpha} \sqcup \bigsqcup_{\alpha' = 1}^{N_K}\bigsqcup_{\beta' = 1}^{q_{\alpha'}} \Theta^{\beta',\alpha'},
\]
 we denote by 
\[
U(G) = \prod_{\alpha = 1}^{N_J} \prod_{\beta=1}^{p_\alpha} U_{(\Xi^{\beta, \alpha}_1)_2}\qquad\text{and}\qquad  V(G) = \prod_{\alpha' = 1}^{N_K} \prod_{\beta'=1}^{q_{\alpha'}} V_{(\Theta^{\beta', \alpha'}_1)_2}.
\]
(Recall that $(\Xi^{\beta, \alpha}_1)_2$ and $(\Theta^{\beta', \alpha'}_1)_2$ denote the starting vertex of the chains $\Xi^{\beta, \alpha}$ and $\Theta^{\beta',\alpha'}$ respectively.) Then, we can rewrite \eqref{eq:Em} in terms of $\mathfrak{G}$ by expanding the products as
\[
\mathcal{E}_m = m^{-\frac{N}{2}}\mathbb{E}\left[\sum_{G\in \mathfrak{G}} \left( \prod_{e\in G} Z_e \right) U(G) V(G)\right].
\]
(Recall \eqref{eq:defN}.) By denoting $\text{mult}_G(e)$ the multiplicity of an edge $e$ in $G$, we can define 
\[
\mathfrak{G}_2 := \{G\in \mathfrak{G} : \text{mult}_G(e) \ge 2\quad\text{for all}\quad e\in G\},
\]
that is, the subset of $\mathfrak{G}$ whose graphs have edges all with multiplicity 2 or higher. By linearity of the expected value, and the fact that all $Z_{ij}$, $U_i$, $V_j$ are independent between them and with average zero, we immediately have that, in fact, we can sum only over $\mathfrak{G}_2$, 
\[
\mathcal{E}_m = m^{-\frac{N}{2}}\mathbb{E}\left[\sum_{G\in \mathfrak{G}_2} \left( \prod_{e\in G} Z_e \right) U(G) V(G)\right].
\]

Let us denote, for any $G\in \mathfrak{G}_2$, $\#v(G)$ the number of different vertices seen by the edges in $G$. In particular, since each edge appears twice for $G\in \mathfrak{G}_2$, we have that 
\begin{equation}
\label{eq:vG}
\#v(G) \le \#v_{\rm end} +\frac{N}{2},
\end{equation}
where we are using that the last $\#v_{\rm end} $ vertices are fixed, that we can add edges (from the end) in such a way that they always see at most one new vertex (since $G$ is connected), and that each edge appears at least twice.

Notice that we have equality in \eqref{eq:vG}   only if each edge in $G$ has multiplicity exactly 2 (otherwise, we would be seeing less vertices than the maximum possible; at some point adding one edge on $G$ would neither contribute to a new vertex nor be a first time repetition): 
\begin{equation}
\label{eq:vG_2}
\#v(G) = \#v_{\rm end} +\frac{N}{2} \quad \Rightarrow \quad \text{for all } e\in G,~~\text{mult}_G(e) = 2. 
\end{equation}

Let us define $\mathfrak{G}_M$ as the subset of graphs in $\mathfrak{G}_2$ that see the maximum number of vertices,
\[
\mathfrak{G}_M := \left\{G\in \mathfrak{G}_2 : \#v (G) = \#v_{\rm end} + \frac{N}{2}\right\},
\]
and let us compute $|\mathfrak{G}_2\setminus \mathfrak{G}_M|$. The elements in $\mathfrak{G}_2\setminus \mathfrak{G}_M$ are all bipartite graphs $G$ between $I_0$ and $I_1$ with $\#v (G)  < \#v_{\rm end} + \frac{N}{2}$ vertices. Since the last  $\#v_{\rm end}$ vertices are fixed, the number of elements in $\mathfrak{G}_2\setminus \mathfrak{G}_M$ will be upper bounded by the number of ways to choose the remaining $\#v (G)-\#v_{\rm end}$ vertices (among $2m$, that is, $(2m-\#v_{\rm end})\cdot (2m-\#v_{\rm end}-1)\cdot \dots\cdot (2m-\#v(G) +1) \le C m^{\#v(G) -\#v_{\rm end}}$). We are also using that, for each configuration of vertices, there is a bounded number of possible graphs with such vertices that is independent of $m$ (but may depend on $N_J$, $N_K$, etc.). In all, since $\#v (G)-\#v_{\rm end} \le \frac{N}{2} - 1$, 
\[
|\mathfrak{G}_2\setminus \mathfrak{G}_M|\le C m^{\frac{N}{2}-1}
\]
for some $C$ independent of $m$. Using that all the elements $Z_{ij}$, $U_i$, $V_j$, have finite moments, we obtain that
\begin{equation}
\label{eq:combined}
\left|\mathcal{E}_m - m^{-\frac{N}{2}}\mathbb{E}\left[\sum_{G\in \mathfrak{G}_M} \left( \prod_{e\in G} Z_e \right) U(G) V(G)\right]\right|\le \frac{C}{m}
\end{equation}
for some $C$ independent of $m$. 

Let now $G\in \mathfrak{G}_M$ be fixed, a graph with maximal number of vertices, \eqref{eq:vG_2}, with $N$ edges each with multiplicity two. Let us count the edges from the vertices to obtain a further characterization of $G$:

The last $\#v_{\rm end}$ vertices are the ending points of the $p+q := p_1+\dots+p_{N_J}+q_1+\dots+q_{N_K}$ chains, and as such, they are connected to at least $p+q$ edges. From the remaining $\frac{N}{2}$ vertices, let us denote by $\#v_E(G)$ the ones that see exactly two edges (which must be the same edge, repeated twice). Then, $\#v_E(G)\le \frac12 (p+q)$. Indeed, since chains have no loops, the same edge cannot be repeated inside a chain, and elements of $v_E(G)$ are necessarily reached by two different chains (and hence they are a starting point for each one). Thus, these starting points see $2\#v_E(G)\le p+q$ edges (counting with multiplicity). 

Finally, the remaining $\#v(G) - \#v_{\rm end} - \#v_E(G) = \frac{N}{2}-\#v_E(G)$ vertices see at least four edges each one, so that the total amount of edges as seen from the vertices (i.e., the sum over vertices of the number of edges seen by each vertex) is: 
\[
\begin{split}
2N = 2\#e(G) & \ge p+q + 2\#v_E(G) +4  (\#v(G) - \#v_{\rm end} - \#v_E(G))\\
& = p+q - 2\#v_E(G) +2N \ge 2N,
\end{split}
\]
where we are also using that each edge is seen from two vertices. In particular, all the previous inequalities are, in fact, equalities, and there are exactly $p+q$ edges connected to $\#v_{\rm end}$, exactly $\frac12 (p+q)$ vertices that are starting points (seeing only two edges each), and all the remaining vertices see four edges (two edges, each with multiplicity two).

At the level of $G\in \mathfrak{G}_M$ this implies that each chain in its definition is repeated exactly identically twice, and that they never share vertices (except for the final ones). In particular, there is an even number of chains ending at each vertex: if $\mathfrak{G}_M \neq \emptyset$ then all $p_1, \dots, p_{N_J}, q_1, \dots, q_{N_K}$ are even. Observe, also, that this implies that 
\[
\mathbb{E}\left[\left( \prod_{e\in G} Z_e \right) U(G) V(G)\right] = 1\qquad\text{for all}\qquad G\in \mathfrak{G}_M,
\]
where we are using $\mathbb{E}[Z_{ij}^2] = \mathbb{E}[U_i^2] = \mathbb{E}[V_j^2] = 1$. 

A short combinatorial argument combined with \eqref{eq:combined} now gives the desired result: we need to count in how many ways we can produce graphs in $G\in \mathfrak{G}_M$ with the definition \eqref{eq:G} in such a way that each chain is repeated identically twice and they never share non-ending vertices. We choose first the chains, which can be done in $m^{\frac{N}{2}}$ ways at leading order (for each chain we choose the previous vertex starting from the end, so there is always $m-r$ possibilities, where $r$ is a bounded number independent of $m$; and we do so for each of the $N$ edges, each of which is repeated twice). For each family of $p_\alpha$ chains ending at $i_\alpha$ we now have multiple ways to produce the same graph $G\in \mathfrak{G}_M$: for each $p_\alpha$ (and $q_{\alpha'}$) we need to count the number of ways in which a family of $p_\alpha$ (and $q_{\alpha'}$) elements can be divided into couples, and then do the same for each $\alpha$ and $\alpha'$. 

In all, given a family of $2n$ elements with $n\in \N$, there are $\frac{(2n)!}{2^n n!}$ ways to split it into couples: there are $(2n)!$ ways to arrange them in a line and we now split them in order into couples. Since we can change the order within each pair, and we can change the order of the pairs, we are actually generating each possible configuration $2^n n!$ times. The number of ways to split $2n$ elements into couples is then $\frac{(2n)!}{2^n n!} = (2n-1)!!$.

Thus, given a fixed graph $G\in \mathfrak{G}_M$, we have $\prod_{\alpha'}^{N_J}(p_{\alpha}-1)!!\prod_{\alpha'}^{N_K}(q_{\alpha'}-1)!!$ ways to produce the same graph with the previous constructions. Combined with \eqref{eq:combined} and the fact that there are $m^{\frac{N}{2}}$ possible configurations (at leading order) we have
\[
\left|\mathcal{E}_m - \mu_{p_1}\dots\mu_{p_{N_J}}\mu_{q_1}\dots \mu_{q_{N_K}}\right| \le \frac{C}{m}
\]
for some $C$ independent of $m$. 
\end{proof}

\subsection{A recursion property}\label{ssec:recursion} We next show a recursion property for  the family \eqref{eq:JKdef} that will be crucial in the following section (recall the notation from subsection~\ref{ssec:notation}). 

\begin{prop}
\label{prop:recurrence}
The random vectors $\bJ_k^m$ and $\bK_k^m$ satisfy,
\[
\begin{array}{rcll}
{m^{-\frac12}} \bZ^m \bJ_k^m & = &\bJ_{k+1}^m+\bJ_{k-1}^m+\bR^m_{ k},&\qquad\text{if $k\in \N$ is even},\\
{m^{-\frac12}} (\bZ^m)^\top \bJ_k^m & = &\bJ_{k+1}^m+\bJ_{k-1}^m+\bR^m_{ k},&\qquad\text{if $k\in \N$ is odd},\\
{m^{-\frac12}} (\bZ^m)^\top \bK_k^m & = &\bK_{k+1}^m+\bK_{k-1}^m+\bS^m_{ k},&\qquad\text{if $k\in \N$ is even},\\
{m^{-\frac12}} \bZ^m \bK_k^m & = &\bK_{k+1}^m+\bK_{k-1}^m+\bS^m_{k},&\qquad\text{if $k\in \N$ is odd},
\end{array}
\]
for some random vectors $\bR_{k}^m$ and $\bS_{k}^m$ with 
\[
\E\left[\|\bR_{k}^m\|^{2p}_{2p}\right] +\E\left[\|\bS_{k}^m\|^{2p}_{2p}\right] \le C_p<+\infty,
\]
for any $p\in \N$, 
and for some $C$ independent of $m$ (but it might depend on $k$ and $p$). 
\end{prop}
\begin{proof}
Let us do the first equality, the others follow by analogy. Thus, we assume $k\in \N$ is even and we deal with $\bJ_k^m$. 
\[
\begin{split}
({m^{-\frac12}} \bZ^m \bJ_k^m)_j & = m^{-\frac{k+1}{2}} \sum_{i = 1}^m Z_{j,i} \sum_{\Xi \in {\tilde \bXi}^m_{i}(k)} \left(\prod_{\ell = 1}^k Z_{\Xi_\ell}\right)U_{(\Xi_1)_2}\\
& =m^{-\frac{k+1}{2}} \sum_{\Xi \in {\mathring \bXi}^m_{j}(k+1)} \left(\prod_{\ell = 1}^{k+1} Z_{\Xi_\ell}\right)U_{(\Xi_1)_2},
\end{split}
\]
where we are denoting
\[
{\mathring \bXi}^m_{j}(k+1) := \left\{\Xi\cup\{(j, i)\} : \Xi \in \tilde{\bXi}^m_{i}(k)\quad\text{for some}\quad 1\le i \le m\right\}.
\]
That is, we are taking loopless chains starting in $I_1$ and with length $k$, and adding an extra edge towards $j$ at the end. In particular, we can divide:
\[
{\mathring \bXi}^m_{j}(k+1)  = {\tilde \bXi}^m_{j}(k+1) \cup  {\mathring \bXi}^m_{j,*}(k+1)\cup {\mathring \bXi}^m_{j,r}(k+1),
\]
where 
\[
{\mathring \bXi}^m_{j,*}(k+1) := \left\{\Xi \in {\mathring \bXi}^m_{j}(k+1) \quad\text{such that}\quad \Xi_{k+1} = \Xi_{k}\right\},
\]
namely, the added edge was already part of the chain (and since we are adding it to a loopless chain, it must be the last edge); and 
\[
{\mathring \bXi}^m_{j,r}(k+1) := {\mathring \bXi}^m_{j}(k+1)\setminus \left( {\tilde \bXi}^m_{j}(k+1) \cup  {\mathring \bXi}^m_{j,*}(k+1) \right)
\]
those chains where the extra edge is not adding a new vertex, but is not a repeated edge either. Thus, 
\begin{equation}
\label{eq:join1}
({m^{-\frac12}} \bZ^m \bJ_k^m)_j = J_{k+1, j}^m + A^m_{k+1, j} + B^m_{k+1, j},
\end{equation}
where, if we denote $v_1(\Xi)$ for $\Xi \in {\bXi}^m_{j}(k-1)$ the set of vertices in $\Xi$ from $I_1$,
\[
\begin{split}
A^m_{k+1, j} & = m^{-\frac{k+1}{2}} \sum_{\Xi \in {\mathring \bXi}^m_{j, *}(k+1)} \left(\prod_{\ell = 1}^{k+1} Z_{\Xi_\ell}\right)U_{(\Xi_1)_2}\\
& = m^{-\frac{k+1}{2}}  \sum_{\Xi \in {\tilde \bXi}^m_{j}(k-1)} \sum_{i \notin v_1(\Xi)} Z_{j,i}^2\left(\prod_{\ell = 1}^{k-1} Z_{\Xi_\ell}\right)U_{(\Xi_1)_2},\\
B^m_{k+1,j} & = m^{-\frac{k+1}{2}} \sum_{\Xi \in {\mathring \bXi}^m_{j, r}(k+1)} \left(\prod_{\ell = 1}^{k+1} Z_{\Xi_\ell}\right)U_{(\Xi_1)_2}.
\end{split}
\]

Observe now that, on the one hand, using the same arguments as in Theorem~\ref{thm:main_convergence}, we can directly compute 
\begin{equation}
\label{eq:join2}
\mathbb{E}\left[(B^m_{k+1,j})^{2p}\right] \le \frac{C_p}{m},
\end{equation}
for any $p\in \N$, and for some $C$ independent of $m$. (We are using here that in the sum we are only considering elements that do not see the maximal number of vertices.)

On the other hand, we can rewrite  
\begin{equation}
\label{eq:join3}
A^m_{k+1, j} = J_{k-1, j}^m + D_{k+1, j}^m+{\tilde D}_{k+1, j}^m,
\end{equation}
with 
\[
\begin{split}
D_{k+1, j}^m & = m^{-\frac{k+1}{2}}  \sum_{\Xi \in {\tilde \bXi}^m_{j}(k-1)} \sum_{i \notin v_1(\Xi)} (Z_{j,i}^2-1) \left(\prod_{\ell = 1}^{k-1} Z_{\Xi_\ell}\right)U_{(\Xi_1)_2},\\
\tilde D_{k+1, j}^m & = m^{-\frac{k+1}{2}}  \sum_{\Xi \in {\tilde \bXi}^m_{j}(k-1)} |v_1(\Xi)|\left(\prod_{\ell = 1}^{k-1} Z_{\Xi_\ell}\right)U_{(\Xi_1)_2}.
\end{split}
\]

From the same arguments as in Theorem~\ref{thm:main_convergence} (since $|v_1(\Xi)|$ is bounded independent of $m$) we get on the one hand that 
\begin{equation}
\label{eq:join4}
\mathbb{E}\left[(\tilde D^m_{k+1,j})^{2p}\right] \le \frac{C_p}{m^2}
\end{equation}
for any $p\in \N$, and on the other hand, since $Z_{j,i}^2 - 1$ has average zero and is independent of all the other elements in each term of the sum (since $i\notin v_1(\Xi)$), the same type of reasoning done in Theorem~\ref{thm:main_convergence} also gives 
\begin{equation}
\label{eq:join5}
\mathbb{E}\left[(D^m_{k+1,j})^{2p}\right] \le \frac{C_p}{m}
\end{equation}
for any $p\in \N$. 

In all, joining \eqref{eq:join1}-\eqref{eq:join2}-\eqref{eq:join3}-\eqref{eq:join4}-\eqref{eq:join5},
\[
({m^{-\frac12}} \bZ^m \bJ_k^m)_j = J_{k+1, j}^m +J_{k-1, j}^m + R^m_{k, j},
\]
with $\mathbb{E}\left[ (R^m_{k, j})^{2p} \right] \le \frac{C_p}{m} $ for any $p\in \N$, and for some $C$ independent of $m$. Using the symmetry of the problem, we get the desired result.
\end{proof}

\subsection{Multi-dimensional output}
\label{sec:multid}
More generally, we can take $d\in \N$ i.i.d. copies of $\bU$, denoted $\bU^{(1)},\dots\bU^{(d)}$ (also independent of $\bV$) and define
\begin{equation}
\label{eq:JKdef3}
\bU^{(1\dots d)} := 
\begin{pmatrix}
    \bU^{(1)}        &      \dots&\bU^{(d)}
\end{pmatrix}
= 
\begin{pmatrix}
    U_1^{(1)}   & \dots &   U_1^{(d)}     \\ 
    U_2^{(1)}    & \dots & U_2^{(d)}      \\
    \vdots       & \ddots & \vdots
\end{pmatrix}.
\end{equation}

Similarly, we denote $\bU^{(\zeta),m}\in \R^m$ for $1\le \zeta \le d$, and 
\begin{equation}
\label{eq:JKdef4}
\bJ^{(\zeta),m}_{k} := \begin{pmatrix}
    J^{(\zeta),m}_{k,1}        \\
    J^{(\zeta),m}_{k,2}        \\
    \vdots     \\
    J^{(\zeta),m}_{k,m}  
\end{pmatrix}\qquad\text{with}\quad 1\le \zeta \le d,
\end{equation}
where
\begin{equation}
\label{eq:JKdef5}
J^{(\zeta),m}_{k,i}  := \frac{1}{m^{k/2}}\sum_{\Xi \in {\tilde \bXi}^m_{i}(k)} \left(\prod_{\ell = 1}^k Z_{\Xi_\ell}\right)U^{(\zeta)}_{(\Xi_1)_2},\qquad\text{with}\quad 1\le \zeta \le d
\end{equation}
(cf. \eqref{eq:JKdef}). Finally, we also consider, for $k\in \N$, families of independent, identically distributed (infinite) random vectors $\{\bJ_k^{(1)}\}_{k\in \N}$,...,$\{\bJ_k^{(d)}\}_{k\in \N}$ and $\{\bK_k\}_{k\in \N}$ 
\begin{equation}
\label{eq:Jddef}
\bJ^{(1)}_{k} := \begin{pmatrix}
    J^{(1)}_{k,1}        \\
    J^{(1)}_{k,2}        \\
    \vdots    
\end{pmatrix},~~\dots~~,
\bJ^{(d)}_{k} := \begin{pmatrix}
    J^{(d)}_{k,1}        \\
    J^{(d)}_{k,2}        \\
    \vdots    
\end{pmatrix},\qquad\text{and}\qquad \bK_{k} := \begin{pmatrix}
    K_{k,1}        \\
    K_{k,2}        \\
    \vdots      
\end{pmatrix}.
\end{equation}
with $J^{(1)}_{k, i},\dots,J^{(d)}_{k, i}, K_{k, i}\sim \mathcal{N}(0,1)$ and all independent between them. 

Then, Theorem~\ref{thm:main_convergence} also holds for this family as well. That is:
\begin{prop}
\label{prop:main_convergence}
The family of random vectors $\{(\bJ_k^{(1),m}, \dots, \bJ_k^{(d),m},\bK_k^m)\}_{k\in \N}$ defined by \eqref{eq:Jddef} converges, in distribution, to the family $\{(\bJ^{(1)}_k,\dots,\bJ^{(d)}_k, \bK_k)\}_{k\in \N}$ (see Definition~\ref{defi:inf_dist_conv}).
\end{prop}
\begin{proof}
We can follow the same ideas as in the proof of Theorem~\ref{thm:main_convergence}, by interpreting \eqref{eq:Em} in this context. We get again that each chain must be repeated twice, and we are only interested in configurations that see a maximal amount of vertices. Now, however, chains can be repeated coming from different elements of the family, namely, $J_{k, i}^{(\zeta),m}$ and $J_{k, i}^{(\zeta'),m}$ might share a chain for $\zeta \neq \zeta'$, and still see the maximum number of vertices. Those repetitions, however, do not contribute to the expected value \eqref{eq:Em}, since they contain a single term $U_s^{(\zeta)} U_s^{(\zeta')}$ for some $1\le s\le m$, and $\E[U_s^{(\zeta)} U_s^{(\zeta')}] = 0$ for $\zeta \neq \zeta'$. 
\end{proof}

We also see the recurrence in Proposition~\ref{prop:recurrence}:
\begin{prop}
\label{prop:recurrence_d}
The random vectors $\bJ_k^{(\zeta),m}$ satisfy,
\[
\begin{array}{rcll}
{m^{-\frac12}} \bZ^m \bJ_k^{(\zeta),m} & = &\bJ_{k+1}^{(\zeta),m}+\bJ_{k-1}^{(\zeta),m}+\bR_{k}^{(\zeta), m},&\qquad\text{if $k\in \N$ is even},\\
{m^{-\frac12}} (\bZ^m)^\top \bJ_k^{(\zeta),m} & = &\bJ_{k+1}^{(\zeta),m}+\bJ_{k-1}^{(\zeta),m}+\bR_{k}^{(\zeta), m},&\qquad\text{if $k\in \N$ is odd},
\end{array}
\]
for some random vectors $\bR_{ k}^{(\zeta), m}$ with 
\[
\E\left[\|\bR_{k}^{(\zeta), m}\|^{2p}_{2p}\right] \le C_p<+\infty,
\]
for any $p\in \N$, and 
for any $1\le \zeta \le m$, and for some $C$ independent of $m$ (but it might depend on $k$ and $p$). 
\end{prop}
\begin{proof}
Follows by Proposition~\ref{prop:recurrence} applied to each $\zeta \in \{1,\dots,d\}$. 
\end{proof}

For notational convenience, we will denote 
\begin{equation}
\label{eq:denote_matrix}
\bJ_k^m = (\bJ_k^{(1),m}, \dots, \bJ_k^{(d),m})\in \R^{m\times d}. 
\end{equation}

\begin{prop}[Almost-Orthonormality property]
\label{prop:orthonormality}
The family of random vectors $\{(\bJ_k^{m}, \bK_k^m)\}_{k\in \N}$ defined by \eqref{eq:Jddef}-\eqref{eq:denote_matrix} satisfies
\[
\E\left[\left\|\bO_{JJ}^m(k_1, k_2)\right\|^{2p}_{2p}+\left\|\bO_{JK}^m(k_1, k_2)\right\|^{2p}_{2p} + \left\|O^m_{KK}(k_1, k_2) \right\|^{2p}_{2p}\right] \le \frac{C_p}{m},
\]
for any $p\in \N$, and 
for some $C$ depending only on $\max\{k_1, k_2\}$, $d$, and $p$, where\footnote{Here, $\delta_{k_1, k_2}$ is the Kronecker delta: $\delta_{k_1, k_2}=1$ if $k_1=k_2$, and $\delta_{k_1, k_2}=0$ if $k_1\neq k_2$.}
\[
\begin{split}
\R^{d\times d}\ni \bO_{JJ}^m(k_1, k_2) & = \frac{1}{m}(\bJ_{k_1}^{m})^\top  \bJ_{k_2}^{m} -\delta_{k_1, k_2}{\rm Id}_d\\
\R^{d\times 1}\ni \bO_{JK}^m(k_1, k_2) &= \frac{1}{m}(\bJ_{k_1}^{m})^\top \bK_{k_2}^{m} \\
\R\ni O_{KK}^m(k_1, k_2)& = \frac{1}{m}\bK_{k_1}^{m} \cdot \bK_{k_2}^{m} - \delta_{k_1, k_2},
\end{split}
\]
for all $k_1, k_2\in \{1,\dots, m\}$.
\end{prop}
\begin{proof}
Let us show 
\[
\E\left[\left(\frac{1}{m}\bK_{k}^{m}\cdot \bK_{k}^{m}-1\right)^2\right]\le \frac{C}{m},
\]
and the rest follow by analogy. We develop the square to obtain
\[
\begin{split}
\E\left[\left(\frac{1}{m}\bK_{k}^{m}\cdot \bK_{k}^{m}-1\right)^2\right] & = \frac{1}{m^2}\sum_{i,j = 1}^m \E\left[(K^m_{k,i})^2 (K^m_{k,j})^2\right]+1-\frac{2}{m}\sum_{i = 1}^m \E\left[(K^m_{k,i})^2\right]\\
& = \frac{m(m-1)}{m^2}\left(\E\left[(K^m_{k,1})^2\right]\right)^2 + \frac{1}{m}\E\left[(K^m_{k,1})^4\right]\\*
& \qquad +1-2\E\left[(K^m_{k,1})^2\right],
\end{split}
\]
where we are using the symmetry in the definition of $K_{k, i}$. 

From the proof of Theorem~\ref{thm:main_convergence} we have that if $k > 0$
\[
\left|\E\left[(K^m_{k,1})^2\right]-1\right|  \le \frac{C}{m},\qquad\text{and}\qquad
\left|\E\left[(K^m_{k,1})^4\right]-3\right|  \le \frac{C}{m},
\]
from which the first result now follows. In general, again using the proof of Theorem~\ref{thm:main_convergence} we have
 \[
\E\left[\left(\frac{1}{m}\bK_{k}^{m}\cdot \bK_{k}^{m}-1\right)^{2p}\right]\le \frac{C_p}{m},
\]
for any $p\in \N$, which gives the desired result.
\end{proof}

\section{Proof of the main result}\label{sec:main-proof}

Let us now proceed with the proof of our main result. Before doing so, we show an intermediary lemma on the possible growth of exchangeable vectors after multiplication by a random matrix. 

For that, we need the following result on random matrices with Gaussian entries, which can be found, for example, in \cite[Theorem 4.4.5]{vershynin2018high}.

\begin{thm}
\label{thm:rand_mat}
Let $\A$ be a random $m\times m$ matrix with subgaussian independent entries with mean zero. Then there exists a universal constant $C > 0$ such that, for any $t > 0$, 
\[
\|\A\|_M := \sup_{x\in \mathbb{S}^{m-1}} \langle \A x, x\rangle \le C K (\sqrt{m}+t)
\]
with probability at least $1-2e^{-t^2}$. Here $K = \max_{i,j}\|A_{i,j}\|_{\psi_2}$, where $\|X\|_{\psi_2} =\inf\left\{t > 0 : \E\left(e^{{X^2}/{t^2}}\right)\le 2\right\}$. 
\end{thm}

Thanks to Theorem~\ref{thm:rand_mat} we can prove the following, where we recall that $\bZ$ denotes a random matrix with i.i.d. entries of $\mathcal{N}(0, 1)$ (in this case, of size $m\times m$).

\begin{lem}
\label{lem:lemtobeapplied}
Suppose that $\kU \in \R^{m\times d}$ is a random exchangeable array that satisfies 
\[
\E\left[\|\kU\|^2\right] \le C_\circ m^\alpha,\qquad\text{and}\qquad 
\E\left[\|\kU\|^\varrho\right] \le C_\circ m^\beta,
\]
for some $\varrho > 2$, $C_\circ \ge 1$, and some $\alpha, \beta > 0$. We define 
\[
\kU' := \frac{1}{\sqrt{m}} \bZ\kU. 
\]
Then, if we let $\delta > 0$ such that $\varrho > 2+\delta$, we have
\[\begin{split}
&\E\left[\|\kU'\|^2\right]  \le C_\circ C_{\varrho} m^\alpha,\qquad\text{and}\qquad \E\left[\|\kU'\|^{\varrho-\delta}\right]  \le C^{\frac{\rho-\delta}{\rho}}_\circ C_{\varrho, \delta} m^{\beta \frac{\varrho-\delta}{\varrho}},
\end{split}
\]
for some constants $C_\varrho,C_{\varrho, \delta}>0$ independent of $m$, but that might depend on $\varrho$, $\alpha$, and $\beta$ (and also on $\delta$ in the case of $C_{\varrho, \delta}$).
\end{lem}
\begin{proof}
We implicitly consider the random elements in a probability space $(\Omega, \mathcal{F}, \PP)$. We define, for any $i \in \N\cup\{0\}$, 
\[
\Omega_i := \left\{\omega \in \Omega :  (i+1)\left\| \kU(\omega)\right\|_2^2\ge \left\|\kU'(\omega)\right\|_2^2\ge i\left\| \kU(\omega)\right\|_2^2\right\}. 
\]
By Theorem~\ref{thm:rand_mat}, for some $C$ independent of $m$ and $i$, 
\begin{equation}
\label{eq:prev_estimate}
\begin{split}
\PP(\Omega_i)& \le \PP\left(\left\|\bZ\right\|_M^2\ge im\right)\\& 
= \PP\left(\left\|\bZ\right\|_M\ge CK\left(\sqrt{m}+\left({\sqrt{i}}{(CK)^{-1}}-1\right)\sqrt{m}\right)\right)\le Ce^{-cim},
\end{split}
\end{equation}
for some universal constants $C$ and $c$. Now observe that, by H\"older's inequality,
\[
\begin{split}
\E\left[\left\|\kU'\right\|^{2}\right] & = \int_\Omega \left\|\kU'(\omega)\right\|^{2}d\PP(\omega) \le \sum_{i\ge 0} (i+1)\int_{\Omega_i}\left\| \kU(\omega)\right\|^{2}d\PP(\omega)\\
& \le \int_{\Omega_0}\left\| \kU(\omega)\right\|^{2}d\PP(\omega) + \sum_{i\ge 1}(i+1)\left(\int_{\Omega_i}\left\| \kU(\omega)\right\|^{{\varrho}}d\PP(\omega)\right)^{\frac{2}{\varrho}}\left(\PP(\Omega_i)\right)^{\frac{\varrho-2}{\varrho}}.
\end{split}
\]

Using the previous estimate, \eqref{eq:prev_estimate},
\[
\E\left[\left\|\kU'\right\|^{2}\right]  \le \E\left[\left\|\kU\right\|^{2}\right]  +C \sum_{i\ge 1}(i+1) \E\left[\left\| \kU\right\|^{{\varrho}}\right]^{\frac{2}{\varrho}} e^{-cim \frac{\varrho-2}{\varrho}}.
\]

From our assumptions, and for some $C_\varrho$ that depends on $\varrho$,
\[
\E\left[\left\|\kU'\right\|^{2}\right]  \le C_\circ m^\alpha +C (C_\circ m^\beta)^{\frac{2}{\varrho}}  \sum_{i\ge 1}\ (i+1) e^{-cim \frac{\varrho-2}{\varrho}} = C_\circ \left(m^\alpha + C_\varrho m^{ \frac{2\beta}{\varrho}}{e^{-cm\frac{\varrho-2}{\varrho}}}\right),
\]
and hence 
\[
\E\left[\left\|\kU'\right\|^{2}\right]  \le C_\circ C_{\varrho} m^\alpha,
\]
for some possibly different $C_\varrho$.

On the other hand, following the same strategy we get:
\[
\E\left[\left\|\kU'\right\|^{\varrho-\delta}\right]  \le C \sum_{i\ge 0} (i+1)\E\left[\left\| \kU'\right\|^{{\varrho}}\right]^{\frac{\varrho-\delta}{\varrho}} e^{-cim \frac{\delta}{\varrho}}.
\]

From the assumptions again, 
\[
\E\left[\left\|\kU'\right\|^{\varrho-\delta}\right]  \le  C^{\frac{\rho-\delta}{\rho}}_\circ C_{\varrho, \delta}m^{\beta \frac{\varrho-\delta}{\varrho}}\sum_{i\ge 0}(i+1)e^{-cim \frac{\delta}{\varrho}}\le C^{\frac{\rho-\delta}{\rho}}_\circ C_{\varrho, \delta}m^{\beta \frac{\varrho-\delta}{\varrho}},
\]
so we get the desired result.
\end{proof}

We can now prove the main result, Theorem~\ref{thm:main}:

\begin{proof}[Proof of Theorem~\ref{thm:main}]
We use the notation from Section~\ref{sec:theobjects}, in particular, the random arrays $(\bJ_k^{(1),m} , \dots, \bJ_k^{(d),m} )$ and $\bK_k^m$, defined by \eqref{eq:JKdef0}-\eqref{eq:JKdef2}-\eqref{eq:JKdef}-\eqref{eq:JKdef3}-\eqref{eq:JKdef4}-\eqref{eq:JKdef5}, with $\bU^{(1\dots d)}$ and $\bV$ taken to be $\bU(0)$ and $\bV(0)$. 

We divide the proof into seven steps.
\\[0.3cm]
\noindent {\bf Step 1: The structure.} For notational convenience, we drop the subscript $\tau > 0$ and the superscript $m$, which will be implicit in the following variables; also, we denote by $\bJ_k = (\bJ_k^{(1),m} , \dots, \bJ_k^{(d),m} ) \in \R^{m\times d}$. We show by induction that we can write
\begin{equation}
\label{eq:wecanwrite}
\begin{split}
\bU(\kappa) & = \mathring\bU(\kappa)+  \kU(\kappa)\\
\bW(\kappa) & = \brW(\kappa) + \kW(\kappa)\\
\bV(\kappa)&  = \brV(\kappa)+ \kV(\kappa),
\end{split}
\end{equation}
with 
\begin{equation}
\label{eq:wecanwrite2}
\begin{split}
\brU(\kappa) & = \sum_{k\ge 0}\left(\bJ_k \balpha_k(\kappa) + \bK_k \overline{\balpha}_k(\kappa)\right)\\
\brW(\kappa) & = \sum_{i,j\ge 0}\left(\bJ_i \bgamma_{ij}(\kappa)\bJ_j^\top+\bK_i \overline{\gamma}_{ij}(\kappa)\bK_j^\top+\bJ_i \hat{\bgamma}_{ij}(\kappa)\bK_j^\top+\bK_i \doublehat{\bgamma}_{ij}(\kappa)\bJ_j^\top\right)\\
\brV(\kappa)&  = \sum_{k\ge 0}\left(\bJ_k \bbeta_k(\kappa) + \bK_k \overline{\beta}_k(\kappa)\right),
\end{split}
\end{equation}
and where 
\begin{equation}
\label{eq:errors}
\kU(\kappa)\in \R^{m\times d}, \kW(\kappa)\in \R^{m\times m}, \kV(\kappa)\in \R^{m}
\end{equation}
satisfy
\begin{equation}
\label{eq:errors2}
\E\left[\|\kU(\kappa)\|^2+\frac{1}{m}\|\kW(\kappa)\|^2+\|\kV(\kappa)\|^2\right] \le Cm^{\frac12},
\end{equation}
for some $C$ depending on $\kappa$ but independent of $m\in \N$. Moreover,
\begin{equation}
\label{eq:coefficients}
\begin{array}{ll}
\balpha_k(\kappa)\in \R^{d\times d},& \quad \balpha_k(\kappa) = 0\quad\text{if $k$ is odd},\\[0.1cm]
\overline{\balpha}_k(\kappa)\in \R^{1\times d},& \quad \overline{\balpha}_k(\kappa) = 0\quad\text{if $k$ is even},\\[0.1cm]
\bbeta_k(\kappa)\in \R^{d\times 1},& \quad \bbeta_k(\kappa) = 0\quad\text{if $k$ is even},\\[0.1cm]
\overline{\beta}_k(\kappa)\in \R,& \quad \overline{\beta}_k(\kappa) = 0\quad\text{if $k$ is odd},\\[0.1cm]
\bgamma_{ij}(\kappa)\in \R^{d\times d},& \quad \bgamma_{ij}(\kappa) = 0\quad\text{if $i$ is even \underline{or} $j$ is odd},\\[0.1cm]
\overline{\gamma}_{ij}(\kappa)\in \R,& \quad \overline{\gamma}_{ij}(\kappa) =  0\quad\text{if $i$ is odd \underline{or} $j$ is even},\\[0.1cm]
\hat{\bgamma}_{ij}(\kappa)\in \R^{d\times 1},& \quad \hat{\bgamma}_{ij}(\kappa) =  0\quad\text{if $i$ is even \underline{or} $j$ is even},\\[0.1cm]
\doublehat{\bgamma}_{ij}(\kappa)\in \R^{1\times d},& \quad \doublehat{\bgamma}_{ij}(\kappa) =  0\quad\text{if $i$ is odd \underline{or} $j$ is odd}.
\end{array}
\end{equation} 
\noindent {\bf Step 2: Computing the update.} Let us compute $(\bU(\kappa+1), \bW(\kappa+1), \bV(\kappa+1))$ in terms of \eqref{eq:wecanwrite}-\eqref{eq:wecanwrite2}, by using \eqref{eq:training_m}. By the inductive  assumption, we will assume that \eqref{eq:coefficients} holds at time $\kappa$. 

We compute first $h_{\kappa}(x)$, by expanding:
\[
\bp(\kappa) := \left[\frac{1}{\sqrt{m}}\bZ + \frac{1}{m}\bW(\kappa)\right]\bU(\kappa).
\]

On the one hand, thanks to \eqref{eq:wecanwrite}-\eqref{eq:wecanwrite2}-\eqref{eq:coefficients} and the recursion property  in Proposition~\ref{prop:recurrence_d}, we have
\[
\begin{split}
\frac{1}{\sqrt{m}}\bZ\bU(\kappa) & = \sum_{k\ge 0}\left(\left[\bJ_{k+1}+\bJ_{k-1} \right] \balpha_k(\kappa) + \left[\bK_{k+1}+\bK_{k-1} \right] \overline{\balpha}_k(\kappa)\right)+\kE_1(\kappa),
\end{split}
\]
where 
\[
\kE_1(\kappa)  = \sum_{k\ge 0}\left(\bR_k \balpha_k(\kappa) + \bS_k \overline{\balpha}_k(\kappa)\right)+ \frac{1}{\sqrt{m}}\bZ \kU(\kappa),
\]
and  where from now on we assume that whenever an index is negative, the corresponding object is identically zero (e.g. $\bJ_{-1} \equiv 0$ and $\bK_{-1} \equiv 0$), and we have denoted (from Proposition~\ref{prop:recurrence_d}), $\bR_k = (\bR^{(1),m}_k,\dots,\bR^{(d),m}_k)$. 

On the other hand, also from \eqref{eq:wecanwrite}-\eqref{eq:wecanwrite2}, we can compute the other term in $\bp$ by using the orthonormality property in Proposition~\ref{prop:orthonormality},
\[
\begin{split}
\frac{1}{{m}}\bW(\kappa)\bU(\kappa) & = \sum_{i,j\ge 0}\left(\bJ_i \bgamma_{ij}(\kappa)\balpha_j(\kappa)+\bK_i \overline{\gamma}_{ij}(\kappa)\overline{\balpha}_j(\kappa)\right)\\
&\quad +  \sum_{i,j \ge 0} \left(\bJ_i \hat{\bgamma}_{ij}(\kappa)\overline{\balpha}_j(\kappa)+\bK_i \doublehat{\bgamma}_{ij}(\kappa)\balpha_j(\kappa)\right)+ \kE_2(\kappa)
\end{split}
\]
where 
\[
\begin{split}
\kE_2(\kappa) & = \frac{1}{m} \left(\kW(\kappa)  \brU(\kappa) +  \brW(\kappa)  \kU(\kappa) +\kW(\kappa)  \kU(\kappa)\right)+\\
& \quad +\sum_{i,j, k\ge 0}\left[\bJ_i \bgamma_{ij}(\kappa)+\bK_i \doublehat{\bgamma}_{ij}(\kappa)\right]\left[\bO_{JJ}(j,k)\balpha_k(\kappa)+\bO_{JK}(j,k)\overline{\balpha}_k(\kappa)\right]\\
& \quad +\sum_{i,j, k\ge 0}\left[\bJ_i \hat{\bgamma}_{ij}(\kappa)+\bK_i \overline{\gamma}_{ij}(\kappa)\right]\left[(\bO_{JK}(k,j))^\top\balpha_k(\kappa)+O_{KK}(j,k)\overline{\balpha}_k(\kappa)\right],
\end{split}
\]
and thus
\[
\begin{split}
\bp(\kappa) & = \brp(\kappa) + \kE_p(\kappa) := \brp(\kappa) + \kE_1(\kappa)+\kE_2(\kappa),
\end{split}
\]
where
\[
\begin{split}
\brp(\kappa) & = \sum_{k\ge 0} \bJ_k \biggl(\balpha_{k+1}(\kappa)+\balpha_{k-1}(\kappa)+\sum_{j \ge 0}\left[\bgamma_{kj}(\kappa)\balpha_j(\kappa)+\hat{\bgamma}_{kj}(\kappa)\overline{\balpha}_j(\kappa)\right]\biggr)
\\& \quad + \sum_{k\ge 0} \bK_k \biggl(\overline{\balpha}_{k+1}(\kappa)+\overline{\balpha}_{k-1}(\kappa)+\sum_{j \ge 0}\left[\overline{\gamma}_{kj}(\kappa)\overline{\balpha}_j(\kappa)+\doublehat{\bgamma}_{kj}(\kappa){\balpha}_j(\kappa)\right]\biggr).
\end{split}
\]

From here, using again the orthonormality property in Proposition~\ref{prop:orthonormality}, we can compute:
\[
\begin{split}
h_{\kappa}(x) & = \frac{1}{m} (\bV(\kappa))^\top \bp(\kappa) x = \mathring{h}_{\kappa}(x) + \kE_h(\kappa) x
\end{split}
\]
with 
\[
\begin{split}
\mathring{h}_{\kappa}(x) & :=  \sum_{k\ge 0} (\bbeta_k)^\top \biggl(\balpha_{k+1}(\kappa)+\balpha_{k-1}(\kappa)+\sum_{j \ge 0}\left[\bgamma_{kj}(\kappa)\balpha_j(\kappa)+\hat{\bgamma}_{kj}(\kappa)\overline{\balpha}_j(\kappa)\right]\biggr)x
\\& \quad + \sum_{k\ge 0} \overline{\beta}_k \biggl(\overline{\balpha}_{k+1}(\kappa)+\overline{\balpha}_{k-1}(\kappa)+\sum_{j \ge 0}\left[\overline{\gamma}_{kj}(\kappa)\overline{\balpha}_j(\kappa)+\doublehat{\bgamma}_{kj}(\kappa){\balpha}_j(\kappa)\right]\biggr)x
\end{split}
\] 
and
\[
\begin{split}
\kE_h(\kappa) & = \frac{1}{m}(\brV(\kappa))^\top \kE_p(\kappa)+ \frac{1}{m}(\kV(\kappa))^\top \brp(\kappa) + \frac{1}{m}(\kV(\kappa))^\top\kE_p(\kappa)\\
& \quad + \sum_{i, k\ge 0} \left[(\bbeta_i(\kappa))^\top\bO_{JJ}(i, k)+\overline{\beta}_i(\kappa) (\bO_{JK}(k, i))^\top\right]\cdot \\
& \qquad \cdot \biggl(\balpha_{k+1}(\kappa)+\balpha_{k-1}(\kappa)+\sum_{j \ge 0}\left[\bgamma_{kj}(\kappa)\balpha_j(\kappa)+\hat{\bgamma}_{kj}(\kappa)\overline{\balpha}_j(\kappa)\right]\biggr)
\\& \quad   + \sum_{i, k\ge 0} \left[(\bbeta_i(\kappa))^\top\bO_{JK}(i, k)+\overline{\beta}_i(\kappa)O_{KK} (i, k)\right]\cdot\\
& \qquad \cdot  \biggl(\overline{\balpha}_{k+1}(\kappa)+\overline{\balpha}_{k-1}(\kappa)+\sum_{j \ge 0}\left[\overline{\gamma}_{kj}(\kappa)\overline{\balpha}_j(\kappa)+\doublehat{\bgamma}_{kj}(\kappa){\balpha}_j(\kappa)\right]\biggr).
\end{split}
\]

At this point it is important to notice that the expression for $\mathring{h}_{\kappa}(x)$ is independent of the basis, and thus, if $m$ is large enough and $\kappa$ is fixed, it is independent of $m$. 

We can also write an expression for $\bV(\kappa+1)$ using \eqref{eq:training_m} directly, where it is easy to check that $\bV(\kappa+1)$ can be written in the form \eqref{eq:wecanwrite}-\eqref{eq:wecanwrite2} with coefficients satisfying \eqref{eq:coefficients} by induction. 

Using a similar procedure (thanks to Propositions~\ref{prop:orthonormality} and \ref{prop:recurrence_d}) we find the expression for
\[
\begin{split}
\bq(\kappa) & := \left[\frac{1}{\sqrt{m}}\bZ + \frac{1}{m}\bW(\kappa)\right]^\top\bV(\kappa) = \brq(\kappa) + \kE_q(\kappa) := \brq(\kappa) + \kE_3(\kappa)+\kE_4(\kappa),
\end{split}
\]
where
\[
\begin{split}
\brq(\kappa) & := \sum_{k\ge 0} \bJ_k \biggl(\bbeta_{k+1}(\kappa)+\bbeta_{k-1}(\kappa)+\sum_{j \ge 0}\left[(\bgamma_{jk}(\kappa))^\top\bbeta_j(\kappa)+(\doublehat{\bgamma}_{jk}(\kappa))^\top\overline{\beta}_j(\kappa)\right]\biggr)
\\& \quad + \sum_{k\ge 0} \bK_k \biggl(\overline{\beta}_{k+1}(\kappa)+\overline{\beta}_{k-1}(\kappa)+\sum_{j \ge 0}\left[\overline{\gamma}_{jk}(\kappa)\overline{\beta}_j(\kappa)+(\hat{\bgamma}_{jk}(\kappa))^\top{\bbeta}_j(\kappa)\right]\biggr),
\end{split}
\]
and, as before, we have
\[
\kE_3(\kappa) = \sum_{k\ge 0}\left(\bR_k \bbeta_k(\kappa) + \bS_k \overline{\beta}_k(\kappa)\right)+ \frac{1}{\sqrt{m}}\bZ^\top \kV(\kappa),
\]
and
\[
\begin{split}
\kE_4(\kappa) & = \frac{1}{m} \left((\kW(\kappa))^\top  \brV(\kappa) +  (\brW(\kappa))^\top  \kV(\kappa) +(\kW(\kappa))^\top  \kV(\kappa)\right)+\\
& \quad +\sum_{i,j, k\ge 0}\left[\bJ_i (\bgamma_{ji}(\kappa))^\top+\bK_i (\hat{\bgamma}_{ji}(\kappa))^\top\right]\left[\bO_{JJ}(j,k)\beta_k(\kappa)+\bO_{JK}(j,k)\overline{\beta}_k(\kappa)\right]\\
& \quad +\sum_{i,j, k\ge 0}\left[\bJ_i (\doublehat{\bgamma}_{ji}(\kappa))^\top+\bK_i \overline{\gamma}_{ji}(\kappa)\right]\left[(\bO_{JK}(k,j))^\top\bbeta_k(\kappa)+O_{KK}(j,k)\overline{\beta}_k(\kappa)\right].
\end{split}
\]
\noindent {\bf Step 3: The evolution.} We can now use the expressions for $\bp(\kappa)$, $\bq(\kappa)$, and $\bV(\kappa) x^\top (\bU(\kappa))^\top$ to derive an evolution for the coefficients \eqref{eq:coefficients} from \eqref{eq:training_m}. In order to do that, let us observe that we can denote
\[
 \R^{d}\ni\bxi_{\kappa} =   \int x\, \mathcal{L}'(h_{\kappa}(x), y)   d\rho_\kappa(x, y) = \brxi_{\kappa} + \kE_\xi(\kappa), 
\]
with 
\[
\brxi_{\kappa} := \int x\, \mathcal{L}'(\mathring{h}_{\kappa}(x), y)   d\rho_\kappa(x, y)
\]
and
\[ \kE_\xi(\kappa) :=  \int x\, \left( \mathcal{L}'(h_{\kappa}(x), y)-\mathcal{L}'(\mathring{h}_{\kappa}(x), y))\right)   d\rho_\kappa(x, y),
\]
so that, since  $\mathcal{L}'$ is Lipschitz and \eqref{eq:unif_second_moments} holds, 
\[
|\kE_\xi(\kappa)| \le C \|\kE_h(\kappa)\|\int |x|^2 \, d\rho_\kappa(x, y)\le C \|\kE_h(\kappa)\| . 
\]

If we now denote 
\[
\delta \balpha_k(\kappa) := \frac{1}{\tau}  \left( \balpha_k(\kappa+1)-\balpha_k(\kappa)\right)
\]
(analogously for $\overline{\balpha}_k, \bbeta_k, \overline{\beta}_k, \bgamma_{ij}, \overline{\gamma}_{ij}, \hat{\bgamma}_{ij}, \doublehat{\bgamma}_{ij}$) we get, on the one hand,\footnote{Observe that the evolution of the system is ``mostly'' independent of $m$, and hence for $m$ very large we have a trivial limit: the only problem is when all the elements in the corresponding arrays are non-zero, due to a boundary effect at $k = m$, but this is avoided for $m$ large enough and after a fixed number of iterations thanks to the initialization \eqref{eq:initialization_coef}.}
\begin{equation}
\label{eq:simplecheck1}{\small
\begin{split}
\delta \balpha_k(\kappa) & = - \biggl(\bbeta_{k+1}(\kappa)+\bbeta_{k-1}(\kappa)+\sum_{j \ge 0}\left[(\bgamma_{jk}(\kappa))^\top\bbeta_j(\kappa)+(\doublehat{\bgamma}_{jk}(\kappa))^\top\overline{\beta}_j(\kappa)\right]\biggr)\brxi_{\kappa}^\top,\\
\delta \overline{\balpha}_k(\kappa) & = -\bigg(\overline{\beta}_{k+1}(\kappa)+\overline{\beta}_{k-1}(\kappa)+\sum_{j \ge 0}\left[\overline{\gamma}_{jk}(\kappa)\overline{\beta}_j(\kappa)+(\hat{\bgamma}_{jk}(\kappa))^\top{\bbeta}_j(\kappa)\right]\bigg)\brxi_{\kappa}^\top,\\
\delta \bbeta_k(\kappa) & =   - \bigg(\balpha_{k+1}(\kappa)+\balpha_{k-1}(\kappa)+\sum_{j \ge 0}\left[\bgamma_{kj}(\kappa)\balpha_j(\kappa)+\hat{\bgamma}_{kj}(\kappa)\overline{\balpha}_j(\kappa)\right]\bigg)\brxi_{\kappa},\\
\delta \overline{\beta_k}(\kappa) & = - \bigg(\overline{\balpha}_{k+1}(\kappa)+\overline{\balpha}_{k-1}(\kappa)+\sum_{j \ge 0}\left[\overline{\gamma}_{kj}(\kappa)\overline{\balpha}_j(\kappa)+\doublehat{\bgamma}_{kj}(\kappa){\balpha}_j(\kappa)\right]\bigg)\brxi_{\kappa},
\end{split}
}
\end{equation}
and on the other hand, from \eqref{eq:training_m} and \eqref{eq:wecanwrite}-\eqref{eq:wecanwrite2} we immediately have
\begin{equation}
\label{eq:simplecheck2}
\begin{array}{ll}
\delta \bgamma_{ij}(\kappa)  = - \bbeta_i(\kappa)\brxi_{\kappa}^\top\balpha_j(\kappa)^\top, & \delta \hat{\bgamma}_{ij}(\kappa)  = - \bbeta_i(\kappa)(\brxi_{\kappa})^\top(\overline{\balpha}_j(\kappa))^\top,\\[0.2cm]
\delta \overline{\gamma}_{ij}(\kappa)  = - \overline{\beta}_i(\kappa)\brxi_{\kappa}^\top(\overline{\balpha}_j(\kappa))^\top, & \delta \doublehat{\bgamma}_{ij}(\kappa)  = - \overline{\beta}_i(\kappa)\brxi_{\kappa}^\top({\balpha}_j(\kappa))^\top,
\end{array}
\end{equation}
and 
\begin{equation}
\label{eq:errorupdate}
\begin{split}
\delta \kU(\kappa) & =- \kE_q(\kappa) \brxi^\top_{\kappa}-\brq(\kappa) \kE_\xi^\top(\kappa) - \kE_q(\kappa)  \kE_\xi^\top(\kappa)\\
\delta \kW(\kappa) & =- \bV(\kappa)\bxi_{\kappa}^\top(\bU(\kappa))^\top+\brV(\kappa)\brxi_{\kappa}^\top(\brU(\kappa))^\top\\
\delta \kV(\kappa) & =- \kE_p(\kappa) \brxi_{\kappa}-\brp(\kappa) \kE_\xi(\kappa) - \kE_p(\kappa)  \kE_\xi(\kappa) .
\end{split}
\end{equation}

It is now a simple check that \eqref{eq:simplecheck1}-\eqref{eq:simplecheck2}  imply that, if the relations in \eqref{eq:coefficients} hold at time $\kappa$, they also hold at time $\kappa +1$. 
\\[0.3cm]
\noindent {\bf Step 4: Initial conditions and boundedness of coefficients.}
We are considering the vectors $\bJ_k$ and $\bK_k$ to be the ones constructed in subsection~\ref{sec:theobjects} where $\bU$ and $\bV$ are the initializations $\bU(0)$ and $\bV(0)$. Thus, by construction, 
\begin{equation}
\label{eq:initialization_coef}
\begin{array}{lll}
 \balpha_0(0) = {\rm Id}_d, \qquad & \balpha_k(0) = {\0}_{d\times d}&\quad\text{for}\quad k \ge 1,\\
  \overline \balpha_k(0) = \0_{1\times d},& &\quad\text{for}\quad k \ge 0,\\
   \bbeta_k(0) = {\0}_{d\times 1}, &&\quad\text{for}\quad k \ge 0,\\
    \overline{\beta}_0(0) = 1, \qquad  &\overline{\beta}_k(0) = 0&\quad\text{for}\quad k \ge 1,\\
 \end{array}
\end{equation}
and all $\bgamma$, $\overline{\gamma}$, $\hat{\bgamma}$, and $\doublehat{\bgamma}$ are initialized at 0. From the update \eqref{eq:simplecheck1}, we immediately get that 
\begin{equation}
\label{eq:zerocoef1}
\|\balpha_k(\kappa)\|=\|\overline\balpha_k(\kappa)\|=\|\bbeta_k(\kappa)\|=\|\overline\beta_k(\kappa)\|=0\qquad\text{if}\quad k \ge \kappa+1, 
\end{equation}
and 
\begin{equation}
\label{eq:zerocoef2}
\|\bgamma_{ij}(\kappa)\|=\|\overline\bgamma_{ij}(\kappa)\|=\|\hat \gamma_{ij}(\kappa)\|=\|\doublehat \bgamma_{ij}(\kappa)\| = 0\qquad\text{if}\quad i \ge \kappa+1\quad\text{or}\quad j \ge \kappa+1,
\end{equation}
and in particular, if $m$ is large enough, the coefficients are all independent of $m$. This is because in the evolution \eqref{eq:simplecheck1}-\eqref{eq:simplecheck2} each element in a position $k$ is only affected by elements in the surrounding positions (either for $\alpha$, $\beta$, or $\gamma$). 

Observe that, again by construction, the evolution of the coefficients \eqref{eq:coefficients} given by \eqref{eq:simplecheck1}-\eqref{eq:simplecheck2} is deterministic, and in particular all the coefficients are always bounded after finitely many time steps by a universal constant depending only on $\kappa_*$ by \eqref{eq:zerocoef1}-\eqref{eq:zerocoef2} (and $\tau$, but independent of $m$), and the same holds for $\brxi$ and $\mathring{h}$:
\begin{equation}
\label{eq:boundcoef}
\begin{split}
& \|\balpha_k(\kappa)\|+\|\overline\balpha_k(\kappa)\|+\|\bbeta_k(\kappa)\|+\|\overline\beta_k(\kappa)\|+ 
\\
& \quad +\|\bgamma_{ij}(\kappa)\|+\|\overline\bgamma_{ij}(\kappa)\|+\|\hat \gamma_{ij}(\kappa)\|+\|\doublehat \bgamma_{ij}(\kappa)\| +|\mathring{h}_\kappa| +\|\brxi_\kappa\|  \le C_{\kappa_*}
\end{split}
\end{equation}
for all $k, i, j \in \N$, and $1\le \kappa\le\kappa_*$.

 Hence, from \eqref{eq:wecanwrite2} and by the proof of Theorem~\ref{thm:main_convergence}, we also have that 
\begin{equation}
\label{eq:boundmainpart}
\E[|\mathring{U}_{i, \ell}(\kappa)|^{\Upsilon}+|\mathring{W}_{i,j}(\kappa)|^{\Upsilon}+|\mathring{V}_{i}(\kappa)|^{\Upsilon}+|\mathring{p}_{i, \ell}(\kappa)|^{\Upsilon}+|\mathring{q}_{i}(\kappa)|^{\Upsilon}]\le C_{\kappa_*, \Upsilon} < +\infty,
\end{equation}
for any $\Upsilon\ge 2$, $1\le i, j\le m$, $1\le \ell\le m$, and for some $C_{\kappa_*, \Upsilon}$ independent of $m$. 

In particular, we have that 
\begin{equation}
\label{eq:boundmainpart2}
\begin{split}
\E[\|\brU(\kappa)\|^{\Upsilon}+\|\brV(\kappa)\|^{\Upsilon}+\|\brp(\kappa)\|^{\Upsilon}+\|\brq(\kappa)\|^{\Upsilon}] & \le C_{\kappa_*, \Upsilon} m^{\frac{\Upsilon}{2}},
\\
\E[\|\brW(\kappa)\|^{\Upsilon}] & \le C_{\kappa_*, \Upsilon} m^{{\Upsilon}}.
\end{split}
\end{equation}

Let us now bound the error terms. Let us assume that we have for some $\alpha\ge 0$ that will be small, and for any $\varrho \ge 2$, 
\begin{equation}
\label{eq:assump}
\begin{split}
\E[\|\kU(\kappa)\|^\varrho+\|\kV(\kappa)\|^\varrho] &\le C_\varrho m^{\frac{\varrho}{2}-1+\alpha} \\ \E[\|\kW(\kappa)\|^\varrho] &\le C_\varrho m^{{\varrho}-1+\alpha}.
\end{split}
\end{equation}
Then we will show that for any $\delta > 0$
\begin{equation}
\label{eq:thesis}
\begin{split}
\E[\|\kU(\kappa+1)\|^\varrho+\|\kV(\kappa+1)\|^\varrho] &\le C'_{\varrho, \delta} m^{\frac{\varrho}{2}-1+\alpha+\delta} \\
 \E[\|\kW(\kappa+1)\|^\varrho] &\le C'_{\varrho, \delta} m^{{\varrho}-1+\alpha+\delta},
 \end{split}
\end{equation}
for some new constant $C'_{\varrho, \delta}$ that might depend on everything, but it is independent of $m$. 

In order to do that, we look at the different terms in the errors. We can always apply the same strategy to bound them, using our hypotheses \eqref{eq:assump} and that we know explicitly how the errors are being updated, \eqref{eq:errorupdate}. Let us for example show how to bound a representative case that includes all possible behaviors, to obtain a bound like \eqref{eq:thesis} for the term $\kE_p(\kappa+1)$. Namely, we will show that 
\begin{equation}
\label{eq:thesis2}
\E[\|\kE_p(\kappa+1)\|^\varrho] \le C' m^{\frac{\varrho}{2}-1+\alpha+\delta}
\end{equation}
for some $\delta$ arbitrarily small. 

We know that $\kE_p(\kappa+1) = \kE_1(\kappa+1)+\kE_2(\kappa+1)$, let us control them separately. 
\\[0.3cm]
{\bf Step 5: Bound on $\kE_1(\kappa+1)$.} The term $\kE_1(\kappa+1)$ has two parts. The first part is
\[
 \sum_{k\ge 0}\left(\bR_k \balpha_k(\kappa) + \bS_k \overline{\balpha}_k(\kappa)\right).
\]
 In this case, we use the boundedness of coefficients \eqref{eq:boundcoef} together with the fact that, from Proposition~\ref{prop:recurrence_d}, we also have that for $k \le \kappa_*+1$ and any $\Upsilon \ge 2$, 
\[
\E[\|\bR_k\|^{\Upsilon}+\|\bS_k\|^{\Upsilon}]  \le C_{\kappa_*, \Upsilon} m^{\frac{\Upsilon}{2}-1},
\]
(using the equivalence between Euclidean norms, $\|x\|_p \le m^{\frac{1}{p} - \frac{1}{q}} \|x\|_q$ for any $x\in \R^m$ and $p < q$). This gives the desired result without losing any power. 

For the second term in $\kE_1(\kappa+1)$,
\[
\frac{1}{\sqrt{m}}\bZ \kU(\kappa),
\]
we can apply Lemma~\ref{lem:lemtobeapplied} to obtain, on the one hand, 
\[
\E\left[\left\|\frac{1}{\sqrt{m}}\bZ \kU(\kappa)\right\|^2\right]\le C m^{\alpha}
\] 
and on the other hand, for any $r > 2$ and  $\delta > 0$, 
\[
\E\left[\left\|\frac{1}{\sqrt{m}}\bZ \kU(\kappa)\right\|^r\right]\le C m^{\frac{r}{2} - 1+\frac{\delta}{r}+\frac{r}{r+\delta} \alpha}\le C m^{\frac{r}{2} - 1+\delta+\alpha}
\] 
where $C$ now might depend also on $\delta$. 
\\[0.3cm]
{\bf Step 6: Bound on $\kE_2(\kappa+1)$.} The term $\kE_2(\kappa+1)$ also has two parts. Regarding all the terms involving the orthonormal errors coming from Proposition~\ref{prop:orthonormality}, we treat them as in the previous step but using Proposition~\ref{prop:orthonormality} instead of Proposition~\ref{prop:recurrence_d}. Let us then show how to bound the remaining term, 
\[
\frac{1}{m} \left(\kW(\kappa)  \brU(\kappa) +  \brW(\kappa)  \kU(\kappa) +\kW(\kappa)  \kU(\kappa)\right).
\]

We do so separately, for each of the three elements. Let us start with the first term:, by means of Cauchy--Schwarz and H\"older:
\[
\begin{split}
\E\left[\left\|\frac{1}{m} \kW(\kappa)  \brU(\kappa) \right\|^\varrho\right]& \le \frac{1}{m^\varrho}\E\left[\left\| \kW(\kappa)  \right\|^\varrho \left\|\brU(\kappa) \right\|^\varrho\right]\\
& \le \frac{1}{m^\varrho}\left(\E\left[\left\| \kW(\kappa)  \right\|^{(1+\eps)\varrho}\right]\right)^{\frac{1}{1+\eps}}\left(\E\left[\left\| \brU(\kappa)  \right\|^{\eta \varrho}\right]\right)^{\frac{1}{\eta}},
\end{split}
\]
with $\frac{\eps}{1+\eps} = \frac{1}{\eta}$. By hypothesis \eqref{eq:assump} and using \eqref{eq:boundmainpart2} we obtain 
\[
\E\left[\left\|\frac{1}{m} \kW(\kappa)  \brU(\kappa) \right\|^\varrho\right]\le C\frac{1}{m^\varrho} m^{\varrho-1+\frac{\eps+\alpha}{1+\eps}} m^{\frac{\varrho}{2}}= C m^{\frac{\varrho}{2} -1+\frac{\eps+\alpha}{1+\eps}}.
\]
 A similar computation works for the term $\frac{1}{m}\brW(\kappa)  \kU(\kappa)$. Finally, 
\[
\begin{split}
\E\left[\left\|\frac{1}{m} \kW(\kappa)  \kU(\kappa) \right\|^\varrho\right]& \le \frac{1}{m^\varrho}\E\left[\left\| \kW(\kappa)  \right\|^\varrho \left\|\kU(\kappa) \right\|^\varrho\right]\\
& \le \frac{1}{m^\varrho}\left(\E\left[\left\| \kW(\kappa)  \right\|^{2\varrho}\right]\right)^{\frac{1}{2}}\left(\E\left[\left\| \kU(\kappa)  \right\|^{2 \varrho}\right]\right)^{\frac{1}{2}}.
\end{split}
\]
Using our hypotheses in \eqref{eq:assump} we have 
\[
\E\left[\left\|\frac{1}{m} \kW(\kappa)  \kU(\kappa) \right\|^\varrho\right] \le C \frac{1}{m^\varrho}m^{{\varrho}+\frac{-1+\alpha}{2}}  m^{\frac{\varrho}{2}+\frac{-1+\alpha}{2}}\le C m^{\frac{\varrho}{2}-1+\alpha}.
\]
Thus, assuming $\eps < \delta$, we have shown that \eqref{eq:thesis2} holds.

We can do the same with all other terms in $\kU(\kappa+1)$ and $\kV(\kappa+1)$ to obtain the desired result, and a completely analogous argument also works on $\kW(\kappa+1)$. 
\\[0.3cm]
{\bf Step 7: Conclusion.} For $\kappa = 0$, there are no error terms, and in particular \eqref{eq:assump} holds with $\alpha = 0$ (recall $\alpha \ge 0$). We fix $\delta$ universally as $\delta = \frac{1}{2\kappa_*}$, in such a way that, from \eqref{eq:assump}-\eqref{eq:thesis} with $\varrho = 2$, 
\[
\E[\|\kU(\kappa)\|^2+\|\kV(\kappa)\|^2] \le C m^{\frac12} \qquad \text{and}\qquad \E[\|\kW(\kappa)\|^2] \le C m^{\frac{3}{2}},
\]
for all $\kappa\le \kappa_*$ (notice that taking $\delta$ smaller, we can make the powers of $m$ arbitrarily close to 1 and 2 respectively). In particular, by exchangeability of $\kU$, $\kV$, and $\kW$, we have that 
\[
\mkU_{i, \ell}(\kappa), \mkV_{j}(\kappa), \mkW_{ij}(\kappa)\to 0 \quad\text{in $L^2$}\quad \text{as $m\to \infty$},
\] 
with 
 \begin{equation}
 \label{eq:lim} 
 \E[|
\mkU_{i, \ell}(\kappa)|^2+ |\mkV_{j}(\kappa)|^2+ | \mkW_{ij}(\kappa)|^2 ]\le Cm^{-\frac12} \to 0 \quad \text{as $m\to \infty$},
\end{equation}
for all $i, j\in \N$, $1\le\ell\le d$ fixed.

The same analysis also yields that, for every $\eps > 0$ there exists some $C_\eps$ such that
\[
\E[\|\kE_h(\kappa)\|^2] \le C_\eps m^{-1+\eps}\to 0\quad\text{as $m \to \infty$},
\]
so that $h_\kappa(x) \to \ring{h}_\kappa(x)$ almost surely for every $x\in \R^d$, as $m\to \infty$. This gives the almost optimal quantitative convergence of the linear predictor, \eqref{eq:quant_convergence}.

We finish by taking, on the one hand
\[
\A(\kappa) = 
\left(
\begin{array}{c}
\balpha_0(\kappa)\\
\overline{\balpha}_1(\kappa)\\
\balpha_2(\kappa)\\
\overline{\balpha}_3(\kappa)\\
\vdots
\end{array}
\right)\in\R^{\infty\times d},\qquad 
\B(\kappa) = 
\left(
\begin{array}{c}
\overline{\beta}_0(\kappa)\\
{\bbeta}_1(\kappa)\\
\overline{\beta}_2(\kappa)\\
{\bbeta}_3(\kappa)\\
\vdots
\end{array}
\right)\in\R^{\infty},
\]
and 
\[
\G(\kappa) = \begin{pmatrix}
\doublehat{\bgamma}_{00}(\kappa) & \overline{\gamma}_{01}(\kappa) & \doublehat{\bgamma}_{02}(\kappa) &  \dots\\
{\bgamma}_{10}(\kappa) & \hat{\bgamma}_{11}(\kappa) & {\bgamma}_{20}(\kappa) &  \dots\\
\doublehat{\bgamma}_{20}(\kappa) & \overline{\gamma}_{21}(\kappa) & \doublehat{\bgamma}_{22}(\kappa) &  \dots\\
\vdots & \vdots & \vdots & \ddots
\end{pmatrix}\in \R^{\infty\times\infty},
\]
which are well defined independently of $m$, if $m$ is large enough for a fixed $\kappa$. On the other hand, recovering the superscripts $m$ in the notation, we know rom Theorem~\ref{thm:main_convergence} (more precisely, from Proposition~\ref{prop:main_convergence}), that the family of vectors $(\bJ_k^m, \bK_k^m)$ converges, in distribution, to a family of independent, identically distributed (infinite) random vectors \eqref{eq:JKdef_lim}, that we denote $(\bJ_k^\infty, \bK_k^\infty)$. Hence, we can take in \eqref{eq:basis_gamma}
\[
\begin{split}
(\bGamma_1,\bGamma_2,\dots) = (\bJ_0^\infty, \bK_1^\infty, \bJ_2^\infty,\dots),\\
(\tilde \bGamma_1,\tilde \bGamma_2,\dots) = (\bK_0^\infty, \bJ_1^\infty, \bK_2^\infty,\dots),
\end{split}
\]
(notice that these equalities are not elementwise, but rather as matrices; that is, $\bGamma_2 = \bJ_0^{(2),\infty}$ if $d \ge 2$).
Thanks to Proposition~\ref{prop:main_convergence} and \eqref{eq:lim}, we are done. 
\end{proof}

\section{Properties of the infinite-width dynamics}\label{sec:limit-properties}
In this section, we study the behavior of the time-continuous ($\tau\to 0$) version of the limit system~\eqref{eq:lim_evo}-\eqref{eq:infsystem}, as $\tau\downarrow 0$, namely the gradient flow of $\mathcal{E}$~\eqref{eq:limit-objective} with initialization~\eqref{eq:init_inf}.

\subsection{A gradient flow} We start by showing that the time-continuous version of \eqref{eq:lim_evo}, \eqref{eq:infsystem_P} below, is a gradient flow of the energy with respect to the Euclidean norm of the parameters (in particular, in the limiting case $m \to \infty$, it is a gradient flow in $\ell^2$), and that the variation of the squared $\ell_2$-norm of each layer is the same; a property that follows from the $1$-homogeneity of the output w.r.t.~each layer, which is often used in the analysis of linear NNs~\cite{arora2019implicit, du2018algorithmic}. This property is often used in conjunction with a \emph{balanced initialization} assumption~\cite[Eq.~(7)]{arora2019implicit}, which does not hold here, in particular because the middle layer has infinite $\ell_2$-norm at initialization. 

%The existence and uniqueness of a solution to \eqref{eq:infsystem_P} is ensured under the further assumption that $\mathcal{L}'$ is Lipschitz.

\begin{prop}
\label{prop:grad_flow}
Let $m\in \N\cup\{\infty\}$. Let $(\A(t), \G(t), \B(t))$  with $\A(t):[0,\infty)\to \R^{m\times d}$, $\G:[0,\infty)\to \R^{m\times m}$, and $\B(t):[0,\infty)\to \R^{m}$ be a solution to the following ODE system
\begin{equation}
\label{eq:infsystem_P}
\left\{
\begin{array}{rcl}
\dot{\A}(t)  & = & - [\LL + \G(t)]^\top\B(t)\bxi^\top_{t},\\
\dot{\G}(t)& = & -\B(t) \bxi^\top_{t} \A(t)^\top,\\
\dot{\B}(t) & = & -[\LL + \G(t)]\A(t)\bxi_{t},
\end{array}
\right.
\end{equation}
with 
\begin{equation}
\label{eq:htdef}
\bxi_{t} = \int x\mathcal{L}'(h_{t}(x), y) d\rho(x, y)\in \R^d,\qquad h_t(x) = \B(t)^\top[\LL + \G(t)]\A(t) x.
\end{equation}
and $\LL\in \R^{m\times m}$ is a fixed matrix, equal to:
\[
 \R^{m\times m}\ni \LL = (\Lambda_{ij})_{ij} = 
\left\{
\begin{array}{ll}
1 & \quad\text{if $i+d = j$ or $j+1 = i$}\\
0 & \quad\text{otherwise}. 
\end{array}
\right.
\]
Then, \eqref{eq:infsystem_P} is the gradient flow in the $\ell_2$-norm of the energy functional 
\[
\mathcal{E}(\A, \G, \B) := \int\mathcal{L}(\B^\top[\LL + \G]\A x, y)d\rho(x, y). 
\]
In particular, we have
\[
\frac{d}{dt} \int \mathcal{L}(h_t(x), y)d\rho(x, y) \le 0,
\]
and
\begin{equation}
\label{eq:AGB}
\frac{d}{dt} \|\A(t)\|^2 = \frac{d}{dt} \|\B(t)\|^2 = \frac{d}{dt} \|\LL+\G(t)\|^2 = -2\int h_t(x) \mathcal{L}'(h_t(x), y)d\rho(x, y).
\end{equation}
\end{prop}
\begin{proof}
Let us formally compute, using \eqref{eq:infsystem_P}
\[
\frac{d}{dt} \|\A(t)\|^2  = \frac{d}{dt} {\rm tr}\left\{ \A(t)^\top \A(t)\right\} = 2{\rm tr}\left\{ \dot{\A(t)}^\top \A(t)\right\} = -2 \B(t)^\top [\LL+\G(t)] \A(t)\bxi_t.
\]
We can proceed similarly with $\LL+\G(t)$ and $\B(t)$ to get \eqref{eq:AGB}. 

The fact that \eqref{eq:infsystem_P} is the gradient flow in the 2-norm of $\mathcal{E}$ is a direct check. For future convenience, we explicitly compute the dissipation by first obtaining the evolution of $h_t(x)$
\begin{equation}
\label{eq:htprime}
\begin{split}
-\frac{d}{dt} h_t(x) &= \bxi^\top_t \A(t)^\top[\LL+\G(t)]^\top [\LL+\G(t) ]\A(t) x+\B(t)^\top \B(t) \bxi_t^\top\A(t)^\top\A(t) x\\
& \quad + \B(t)^\top[\LL+\G(t)] [\LL+\G(t)]^\top\B(t) \bxi_t^\top x,
\end{split}
\end{equation}
so that denoting $\mathcal{E}(t) := \mathcal{E}(\A(t), \G(t), \B(t))$, 
\[
\begin{split}
\frac{d}{dt}\mathcal{E}(t) & = \int \frac{d}{dt}h_t(x) \mathcal{L}'(h_t(x), y)d\rho(x, y) \\
&  = -\|[\LL+\G(t)]\A(t)\bxi_t\|^2-\|\B(t)\|^2\|\A(t) \bxi_t\|^2 - \|\bxi_t\|^2 \|[\LL+\G(t)]^\top \B(t)\|^2.
\end{split}
\]
All the above computations also work if $m = \infty$, in which case we consider the $\ell^2$ norms of the parameters. 
%On the other hand, by Cauchy-Schwarz applied twice 
%\[
%\begin{split}
%\left|\int h_t(x) \mathcal{L}'(h_t(x), y)d\rho(x, y)\right|^2 & = |\B(t)^\top [\LL+\G(t)] \A(t)\bxi_t|^2 \\
%& \hspace{-2cm}\le \|[\LL+\G(t)]^\top \B(t)\|\cdot \|\A(t)\bxi_t\|\cdot \|[\LL+\G(t)] \A(t)\bxi_t\|\cdot \|\B(t)\|.
%\end{split}
%\]
%Combining the previous two expressions together with the arithmetic-geometric inequality, we obtain
%\[
%\begin{split}
%-\frac{d}{dt} \int \mathcal{L}(h_t(x), y)d\rho(x, y) & \ge 3\left|\int h_t(x) \mathcal{L}'(h_t(x), y)d\rho(x, y)\right|^{\frac{4}{3}}\|\bxi_t\|^{\frac{2}{3}}\\
%& = \frac{1}{48^{\frac13}}\left|\frac{d}{dt}\left(\|\A(t)\|^2+\|\LL+\G(t)\|^2+\|\B(t)\|^2\right)\right|^{\frac{4}{3}}\|\bxi_t\|^{\frac{2}{3}}
%\end{split}
%\]
\end{proof}
\begin{rem}
When $m = \infty$, \eqref{eq:AGB} should be paired with some initial conditions that ensure its finiteness, and since $\Vert \LL\Vert=+\infty$, the third term should be interpreted as
\[
\frac{d}{dt} \left(\|\LL+\G(t)\|^2-\|\LL\|^2\right) = \frac{d}{dt} {\rm tr}\left(\LL^\top \G(t) + \G^\top(t)\G(t)\right).
\]
\end{rem}

\subsection{Selection principle}

Recall that we initialize our system \eqref{eq:infsystem_P} with 
\begin{equation}
\label{eq:init_inf_F}
\A(0)  =
\left(
\begin{array}{c}
{\rm Id}_{d} \\
\0_{d\times 1}\\
\0_{d\times 1}\\
\vdots
\end{array}
\right) 
\in \R^{m\times d},\quad
\B (0) = \left(\begin{array}{c}
1 \\
0\\
0\\
\vdots
\end{array}
\right)\in \R^{m},
\end{equation}
and
\begin{equation}
\label{eq:init_inf2_F}
(\G )_{ij}(0) = 0\quad\text{for}\quad 1\le i,j\le m.
\end{equation}
If we denote 
\[
\bl_t := \A(t)^\top [\LL+\G(t)]^\top \B(t)\in \R^d,
\]
so that $h_t(x) = \bl_t^\top x$, we next show that $\bl_t^\top$ never leaves the span of our data. That is, 
\begin{equation}
\label{eq:itholds}
\bl_t^\top \bv = 0\qquad\text{for all}\quad \bv \in {\rm span}\left({\rm supp}((\pi_x)_\# \rho)\right)^\perp,
\end{equation}
where $\pi_x : \R^d\times \R \to \R^d$ is the projection operator $(x, y)\mapsto x$, and $(\pi_x)_\#\rho$ denotes the pushforward of $\rho$ through $\pi_x$.

\begin{prop}
\label{prop:select_principle}
Under the assumptions of Proposition~\ref{prop:grad_flow}, let us further assume that $\mathcal{L}''$ is bounded and that $(\A(t), \G(t), \B(t))$ are initialized at \eqref{eq:init_inf_F} and \eqref{eq:init_inf2_F}. Then \eqref{eq:itholds} holds for all $t \ge 0$. 
\end{prop}
\begin{proof}
Since $\mathcal{L}'$ is Lipschitz we know that the evolution is globally defined in time. Moreover, since $\bl_0 = 0$ we only need to show 
\[
\frac{d}{dt} \bl_t^\top\bv = 0\quad\text{for all}\quad t > 0,
\]
where $\bv \in {\rm span}\left({\rm supp}((\pi_x)_\# \rho)\right)^\perp$ will be fixed throughout the proof. 

We can compute, using \eqref{eq:infsystem_P}, 
\[
\begin{split}
\frac{d}{dt} \bl_t^\top  & = - \B^\top[\LL+\G][\LL+\G]^\top \B\bxi^\top - \B^\top\B \bxi^\top \A^\top\A - \bxi^\top\A^\top[\LL+\G]^\top[\LL+\G]\A,
\end{split}
\]
where we have omitted the time dependence for the sake of readability, that will be made only explicit at time $0$. Observe now that, since $\bv \in {\rm span}\left({\rm supp}((\pi_x)_\# \rho)\right)^\perp$, 
\[
\bxi^\top \bv = 0\qquad\text{for all}\quad t\ge 0. 
\]

Hence, $\dot{\A} \bv = 0$ for all $t\ge 0$, which implies that $\A \bv = \A_0 \bv$ (where $\A_0 = \A(0)$, given by \eqref{eq:init_inf_F}). In all, we have
\begin{equation}
\label{eq:rewrite_MNOP}
\frac{d}{dt}\bl_t^\top \bv = - \B^\top\B \bxi^\top \A^\top\A_0\bv  - \bxi^\top\A^\top[\LL+\G]^\top[\LL+\G]\A_0\bv.
\end{equation}

Let us now define the following quantities:
\[
\begin{array}{rclrcl}
M_t & := &\B^\top \LL \A_0\bv\in \R,&\quad \bN_t & :=& \G^\top\LL \A_0\bv\in \R^m,\\
\bO_t & := &\bPi^\top \A^\top \A_0\bv\in \R^d,&\quad \bP_t & :=& \G \A_0\bv\in \R^m, 
\end{array}
\]
where we have denoted by $\bPi\in \R^{d\times d}$ the projection matrix to ${\rm span}\left({\rm supp}((\pi_x)_\# \rho)\right)$, so that 
\[
\bPi \w = \bw\quad\text{for all}\quad \bw\in {\rm span}\left({\rm supp}((\pi_x)_\# \rho)\right). 
\]
In particular, we always have that $\bxi^\top \bPi^\top = \bxi^\top$.  In the following, we will use that 
\[
\A_0 \bv = \left(
\begin{array}{c}
\bv\\
\0_{ m-d }
\end{array}\right),
\]
and hence, since the first $d\times d$ submatrix of $\LL^\top \LL$ is the identity (which is a simple check) we have
\begin{equation}
\label{eq:LLid}
\LL^\top\LL\A_0\bv = \A_0\bv. 
\end{equation}
A computation using \eqref{eq:infsystem_P} and \eqref{eq:LLid} gives then the following system of ODEs:
\begin{equation}
\label{eq:systMNOP}
\left\{
\begin{split}
\dot{M_t} & = -\bxi^\top\A^\top [\LL+\G]^\top\LL \A_0\bv = -\bxi^\top \bO_t -\bxi^\top\A^\top \bN_t \\
\dot{\bN_t} & = -\A^\top\bxi\B^\top \LL \A_0\bv = -\A^\top\bxi M_t\\
\dot{\bO_t} & = -\bPi^\top\bxi \B^\top [\LL+\G] \A_0\bv = -\bPi^\top\bxi M_t - \bPi^\top\bxi \B^\top \bP_t\\
\dot{\bP_t} & = -\B \bxi^\top \A^\top \A_0\bv = -\B \bxi^\top \bO_t,
\end{split}
\right.
\end{equation}
which is initialized at 
\begin{equation}
\label{eq:systMNOPinit}
M_0 = 0,\quad \bN_0 = \0_m,\quad \bO_0 = \0_d,\quad \bP_0 = \0_m. 
\end{equation}
Here, we used that $\G(0) = 0$, that the first element of $\LL \A_0\bv$ is zero (and hence, $M_0 = 0$), that $\bxi^\top \bPi^\top = \bxi^\top$, and that $\A_0^\top \A_0 = {\rm Id}_{d }$ so $\bO_0 = \bPi^\top \bv = \0_d$.  
The system \eqref{eq:systMNOP} is Lipschitz in its variables, coupled with locally bounded coefficients (thanks to \eqref{eq:AGB}), and therefore it has a unique solution. Since the initial conditions~\eqref{eq:systMNOPinit} all vanish, the unique solution is $(M_t, \bN_t, \bO_t, \bP_t)=(0,\0_m,\0_d,\0_m)$ for $t\ge 0$. 

Finally, we can rewrite \eqref{eq:rewrite_MNOP} in terms of  $(M_t, \bN_t, \bO_t, \bP_t)$ (recalling \eqref{eq:LLid}) as
\[
\frac{d}{dt}\bl_t^\top \bv = - \B^\top\B \bxi^\top \bO_t - \bxi^\top \bO_t -\bxi^\top\A^\top \bN_t - \bxi^\top\A^\top [\LL+\G]^\top \bP_t = 0,
\]
which is our desired result. 
\end{proof}

\begin{rem}
We highlight that the selection principle in Proposition~\ref{prop:select_principle} is not a consequence of a general abstract result on gradient flows with this particular structure, but rather follows from the precise initialization that arises from the infinite width limit, as illustrated by the following example. 

By denoting $\be_1 = \begin{pmatrix} 1 & 0 \end{pmatrix}^\top$, let us define
\[
\mathcal{E}(\A, \z) := \frac12 \langle \A\z, \be_1\rangle^2,\quad\text{with}\quad \R^{2\times 2} \ni \A = \begin{pmatrix} A_{11} & A_{12}\\ A_{21} & A_{22}\end{pmatrix}, \ \R^2 \ni \z = \begin{pmatrix} z_1\\z_2\end{pmatrix},
\]
which is the empirical risk of a two-layer linear NN with a single sample $(x_1,y_1)=( \be_1 ,0)$ in the training set. Consider its gradient flow:
\[
\begin{split}
\dot{\A} & = - \partial_{\A} \mathcal{E}(\A, \z) = -\langle \A\z, \be_1\rangle \be_1 \z^\top = -(A_{11}z_1+ A_{12} z_2) \begin{pmatrix}z_1 & z_2 \\ 0 & 0 \end{pmatrix}, \\
\dot{\z} & = - \partial_{\z} \mathcal{E} = -\langle \A\z, \be_1\rangle \A^\top \be_1 -(A_{11}z_1+ A_{12} z_2) \begin{pmatrix} A_{11} \\ A_{12}\end{pmatrix}. 
\end{split}
\]
Then, if we denote $\bl := \A \z = \begin{pmatrix}\lambda_1& \lambda_2\end{pmatrix}^\top$, we can express the energy as 
\begin{equation}
\label{eq:Elambda}
\mathcal{E}(\A, \z) = \frac12 \langle \A\z, \be_1\rangle^2 = \frac12 \lambda_1^2.
\end{equation}
It is however not true that the evolution of $\bl$ must be such that it always moves along the span of $\be_1$. Indeed, using the previous gradient flow, we know that 
\[
\dot{\bl} = \dot{\A} \z + \A\dot{z} = -(A_{11}z_1+ A_{12} z_2)\left(\begin{pmatrix} z_1^2+z_2^2\\ 0 \end{pmatrix} + \begin{pmatrix} A_{11}^2+A_{12}^2\\ A_{21}A_{11}+A_{22}A_{12} \end{pmatrix} \right).
\]
Hence, when $(A_{11}z_1+ A_{12} z_2)( A_{21}A_{11}+A_{22}A_{12}) \neq 0$, the second coordinate $\bl$ is moving. This can happen  by choosing at time $t = 0$
\[
\A(0) := \begin{pmatrix}1 & 1 \\ 1 & 0\end{pmatrix},\qquad \z(0) := \begin{pmatrix} 0 \\ 1\end{pmatrix},\quad\text{so that}\quad \bl(0) = \begin{pmatrix} 1 \\ 0\end{pmatrix}
\]
 and, since $\dot{\lambda_2}(0) \neq 0$, we have that $\lambda_2(t) \neq 0$ for some time  $t > 0$, despite the fact that the energy in \eqref{eq:Elambda} depends only on $\lambda_1$. 
\end{rem}

\subsection{Quantitative convergence and implicit bias} Whenever the loss function is uniformly convex (we take the quadratic case for convenience) then we expect exponential rate of convergence towards a minimizer.  

In the following, given a measure $\rho$, we denote by $\M$ the covariance matrix, 
\begin{equation}
\label{eq:Mdef}
\M := \int x x^\top d\rho(x, y)\in \R^{d\times d}.
\end{equation}
Note that $\M$ is symmetric and positive semi-definite. In particular, if $\M$ is non-degenerate (${\rm det}(\M) > 0$), then there is a unique minimizer $\bl\in \R^d$ of the quadratic energy
\[
\mathcal{E} = \int(\bl\cdot x - y)^2d\rho(x, y). 
\]
Otherwise, and as we have seen in Proposition~\ref{prop:select_principle}, our system will converge to a minimizer in the span of ${\rm supp}((\pi_x)_\#\rho)$ (alternatively, in ${\rm ker}(\M)^\perp$ or in the row space of $\M$), which is unique. We prove that it will do so at an exponential rate, depending on the lowest non-zero eigenvalue of $\M$.

\begin{prop}
\label{prop:expconvergence}
Under the assumptions of Proposition~\ref{prop:grad_flow}, let us further assume that $\mathcal{L}$ is the quadratic loss function and that $\M$ has $d'$ non-zero eigenvalues, with $1\le d'\le d$, that we denote $0< z_1\le z_2\le \dots\le z_{d'}$. 

Let $(\A(t), \G(t), \B(t))$ denote the evolution \eqref{eq:infsystem_P} initialized at \eqref{eq:init_inf_F} and \eqref{eq:init_inf2_F}. Then $\bl_t$ converges to the unique minimizer $\bl \in \R^d$ of the energy functional,
\[
\mathcal{E}_t := \int(h_t(x) - y)^2 d\rho(x, y) = \int(\bl_t\cdot x - y)^2 d\rho(x, y)
\]
such that $\bl \in {\rm span}\left({\rm supp}((\pi_x)_\#\rho)\right)$ (alternatively, $\bl\in {\rm ker}(\M)^\perp$), and  
\[
\mathcal{E}_t - \mathcal{E}_\infty \le \left(\mathcal{E}_0 - \mathcal{E}_\infty\right)e^{-\tilde c_\lambda t}\quad\text{for}\quad t \ge 0
\]
for some constant $\tilde c_\lambda$ depending only on $\|\bl\|$, $d$, $z_1$, and $z_{d'}$ (and independent of $m$). 
\end{prop}
\begin{proof}
We divide the proof into four steps. 
\\[0.3cm]
{\bf Step 1: The setting.} We use the same notation as in Proposition~\ref{prop:grad_flow} and Proposition~\ref{prop:select_principle}. We recall that we had denoted
\[
 \bl_t := \A(t)^\top [\LL+\G(t)]^\top \B(t)\in \R^d. 
\]
(In particular, $\bl_0 = \0_{\rm d\times 1}$.) The condition on $\M$ can then be re-written as 
\begin{equation}
\label{eq:M_nondeg}
0 < z_1 |\bw|^2\le  \bw\cdot\M\bw \le z_{d'}|\bw|^2\qquad\text{for all}\quad \bw\in {\rm ker}(\M)^\perp. 
\end{equation}
The energy is given by 
\[
\mathcal{E}_t := \mathcal{E}(\A(t), \G(t), \B(t)) = \int (h_t(x) - y)^2\, d\rho(x, y),
\]
where we recall that $h_t(x) = \bl_t\cdot x$. In particular, we can explicitly compute the minimizer $\bl$ (with $\bl\in {\rm ker}(\M)^\perp$) and the evolution of $\mathcal{E}_t$ in terms of $\bl$, 
\begin{equation}
\label{eq:E1}
\bl := \int y \,\M^{-1}x \,\rho(x, y) \in \R^d,\qquad\mathcal{E}_t = (\bl_t - \bl) \cdot \M (\bl_t - \bl) + \mathcal{E}_\infty,
\end{equation}
where, by an abuse of notation, we denoted by $\M^{-1} x$ the inverse restricted to ${\rm ker}(\M)^\perp$ of $x\in {\rm supp}(\pi_x)_\#\rho$, so that $\bl\in {\rm ker}(\M)^\perp$ as well. 
From \eqref{eq:M_nondeg} and the fact that $\bl_t\in {\rm ker}(\M)^\perp$ for all $t\ge 0$ (see Proposition~\ref{prop:select_principle}), we have
\begin{equation}
\label{eq:E2}
z_1 \|\bl_t - \bl\|^2\le \mathcal{E}_t - \mathcal{E}_\infty \le z_{d'}\|\bl_t- \bl\|^2. 
\end{equation}
We also have (cf. \eqref{eq:htprime})
\[
\bxi_t = 2 \M (\bl_t - \bl)\qquad \text{and}\qquad\dot{\bl}_t = -2 \bR_t \M (\bl_t - \bl)
\]
where $\bR_t$ is a symmetric matrix, $\bR_t \ge 0$, defined by 
\begin{equation}
\label{eq:Rt}
\begin{split}
\bR_t  &= \A(t)^\top[\LL+\G(t)]^\top [\LL+\G(t) ]\A(t) +\B(t)^\top \B(t) \A(t)^\top\A(t) \\
& \quad + \B(t)^\top[\LL+\G(t)] [\LL+\G(t)]^\top\B(t) {\rm Id}_{d\times d}\in \R^{d\times d}.
\end{split}
\end{equation}
Thus, 
\begin{equation}
\label{eq:E3}
\dot{\mathcal{E}}_t = -4(\bl_t - \bl) \cdot \M \bR_t \M (\bl_t - \bl). 
\end{equation}

Observe also that (see the proof of Proposition~\ref{prop:grad_flow}) 
\begin{equation}
\label{eq:boundderAB}
\begin{split}
\frac{d}{dt} \|\A(t)\|^2 = \frac{d}{dt} \|\B(t)\|^2 & = -4\bl_t\cdot \M (\bl_t - \bl) \\
& = - 4 (\mathcal{E}_t - \mathcal{E}_\infty) -4\bl\cdot \M (\bl_t - \bl)\\
& \le 4z_{d'}^{\frac12}\|\bl\| \sqrt{\mathcal{E}_t - \mathcal{E}_\infty}\le 4 z_{d'}\|\bl\|^2,
\end{split}
\end{equation}
where we used that the energy is decreasing, Cauchy-Schwarz, and \eqref{eq:M_nondeg}. Similarly, for any $\be\in \mathbb{S}^{d-1}$, 
\begin{equation}
\label{eq:boundderAB2}
\begin{split}
\frac{d}{dt}\|\A(t) \be\|^2 & = -4 (\bl_t \cdot \be )\, \be \M(\bl_t - \bl)\\
& \le  C \|\bl_t -\bl\|^2 +  \|\bl\|\, \|\bl_t - \bl\|\le C\|\bl\|^2
\end{split}
\end{equation}
for some constant $C$ depending only on $z_1$ and $z_{d'}$. 
\\[0.3cm]
{\bf Step 2: Small times.} We have $\bR_t \ge \|\B(t)\|^2 \A(t)^\top \A(t)$ and $\bR_0 \ge {\rm Id}_{d\times d}$. In particular, thanks to \eqref{eq:boundderAB}-\eqref{eq:boundderAB2},
\begin{equation}
\label{eq:E4}
\bR_t \ge \frac12 {\rm Id}_{d\times d}\quad\text{for}\quad t \le t_\circ,
\end{equation}
where $t_\circ = c_\circ\|\bl\|^{-2}$ for some $c_\circ>0$ depending only on $z_1$ and $z_{d'}$. Hence, 
\[
\dot{\mathcal{E}}_t \le - c (\mathcal{E}_t - \mathcal{E}_\infty)\qquad\text{for}\qquad 0 \le t < t_\circ,
\]
for some $c$ depending only on $z_1$ and $z_{d'}$, thanks to \eqref{eq:E1}-\eqref{eq:E2}-\eqref{eq:E3}-\eqref{eq:E4} (we use that if $\M$ and $\bR_t$ are symmetric positive semi-definite matrices, then $\M^{\frac12}\bR_t\M^{\frac12}$ is positive semi-definite as well). In particular, 
\begin{equation}
\label{eq:smalltimes}
\mathcal{E}_t -\mathcal{E}_\infty \le (\mathcal{E}_0 -\mathcal{E}_\infty) e^{-ct}\quad\text{for}\quad 0 \le t < t_\circ.
\end{equation}
\noindent {\bf Step 3: An ODE for all times.} From the previous inequality and the dissipation of energy, we 
have
\[
\|\M^\frac12\bl\|e^{-\frac{ct_\circ}{2}}\ge \|\M^\frac12 (\bl_{t}-\bl)\|\ge \|\M^\frac12 \bl\|- \|\M^\frac12 \bl_{t}\| \quad\text{for}\quad t \ge t_\circ,
\]
so that
\[
\|\bl_{t}\|^2 \ge {C_\rho^{-1}} \|\M^\frac12 \bl_{t}\|^2 \ge {C_\rho^{-1}}\left(1-e^{-\frac{ct_\circ}{2}}\right)^2\|\M^\frac12 \bl\|^2 \ge {C_\rho^{-2}}\|\bl\|^2\left(1-e^{-\frac{ct_\circ}{2}}\right)^2 =:c_\lambda
\]
with $c_\lambda > 0$, for $t \ge t_\circ$. 
 In particular,  by Cauchy-Schwarz and up to a dimensional constant, from the definition of $\bl_t$, 
\[
c_\lambda \le \|\bl_t\|^2 \le C \|\A(t)\|^2 \|[\LL+\G(t)]^\top\B(t)\|^2\quad\text{for}\quad t \ge t_\circ.  
\]
From \eqref{eq:Rt} we know that for some dimensional $c > 0$,
\[
\bR_t \ge \|[\LL+\G(t)]^\top\B(t)\|^2 \,{\rm Id}_{d\times d} \ge c c_\lambda \|\A(t)\|^{-2}\,{\rm Id}_{d\times d}\quad\text{for}\quad t \ge t_\circ  .  
\]

On the other hand, from \eqref{eq:boundderAB}, and since $\|\A(0)\|^2 = d$,
\[
\|\A(t)\|^2 \le d+ C\|\bl\|\int_0^t   \sqrt{\mathcal{E}_\tau - \mathcal{E}_\infty}\, d\tau, 
\]
and hence 
\[
\bR_t \ge \frac{c c_\lambda}{d+ C\|\bl\|\int_0^t   \sqrt{\mathcal{E}_\tau - \mathcal{E}_\infty}\,  d\tau}\, {\rm Id}_{d\times d}\quad\text{for}\quad t \ge t_\circ .  
\]
Combined again with \eqref{eq:M_nondeg}-\eqref{eq:E2}-\eqref{eq:E3} we obtain the inequality
\begin{equation}
\label{eq:ODE_E}
\dot{\mathcal{E}}_t \le - \frac{c c_\lambda (\mathcal{E}_t - \mathcal{E}_\infty)}{1+ \|\bl\|\int_0^t   \sqrt{\mathcal{E}_\tau - \mathcal{E}_\infty}\,  d\tau}\, \quad\text{for}\quad t \ge t_\circ.
\end{equation}
\noindent {\bf Step 4: Bootstrap argument.} Observe that 
\begin{equation}
\label{eq:ODE_E2}
\int_0^t  \sqrt{\mathcal{E}_\tau- \mathcal{E}_\infty}\, d\tau \le C \|\bl\| t,
\end{equation}
since we have dissipation of the energy. Hence, from \eqref{eq:ODE_E} we get 
\[
\dot{\mathcal{E}}_t \le - \frac{c c_\lambda (\mathcal{E}_t - \mathcal{E}_\infty)}{1+ \|\bl\|t}\, \quad\text{for}\quad t \ge t_\circ,
\]
which implies (also using that $c_\lambda \le C \|\bl\|^2$ and $t_\circ = c\|\bl\|^{-2}$)
\[
\mathcal{E}_t - \mathcal{E}_\infty \le \left(\mathcal{E}_{t_\circ} - \mathcal{E}_\infty\right)\left(\frac{1+\|\bl\| t_\circ}{1+\|\bl\| t}\right)^{\frac{cc_\lambda}{\|\bl\|}} \le {C\|\bl\|^2}{\left(1+\|\bl\| t\right)^{-\frac{cc_\lambda}{\|\bl\|}}} ,\quad\text{for}\quad t \ge t_\circ.
\]
Plugging it back into \eqref{eq:ODE_E}, we now have that instead of \eqref{eq:ODE_E2} (also using \eqref{eq:smalltimes}),
\[
\int_0^t \sqrt{\mathcal{E}_\tau- \mathcal{E}_\infty}\, d\tau \le C\frac{\|\bl\|}{1+\|\bl\|^2}+ C\|\bl\|(1+\|\bl\| t)^{1-\eps_\lambda}
\]
where we have denoted $\eps_\lambda := \frac{cc_\lambda}{\|\bl\|} < \frac12$ (if $c$ is sufficiently small). Again from \eqref{eq:ODE_E},
\[
\frac{\dot{\mathcal{E}}_t}{\mathcal{E}_t - \mathcal{E}_\infty} \le - \frac{cc_\lambda}{1+ \|\bl\|^2(1+\|\bl\|^2)^{-2}+ \|\bl\|^2(1+\|\bl\| t)^{1-\eps_\lambda}}\quad\text{for}\quad t \ge t_\circ.
\]
In particular, there exists some $\tilde c_\lambda$ depending on $\|\bl\|$, $z_1$, and $z_{d'}$, such that 
\[
\mathcal{E}_t - \mathcal{E}_\infty \le \left(\mathcal{E}_0 - \mathcal{E}_\infty\right)e^{-\tilde c_\lambda t^{\eps_\lambda}}\quad\text{for}\quad t \ge 0. 
\]
Iterating again the procedure, now $\int_0^\infty\sqrt{\mathcal{E}_\tau - \mathcal{E}_\infty}\, d\tau < +\infty$, and hence 
\[
\mathcal{E}_t - \mathcal{E}_\infty \le \left(\mathcal{E}_0 - \mathcal{E}_\infty\right)e^{-\tilde c_\lambda t}\quad\text{for}\quad t \ge 0
\]
for some (possibly different) $\tilde c_\lambda$ depending only on $\|\bl\|$, $d$, $z_1$, and $z_{d'}$ 
\end{proof}
 Finally, we have:
 \begin{proof}[Proof of Theorem~\ref{thm:main2}]
 If follows from Proposition~\ref{prop:expconvergence}. 
 \end{proof}

\section{Multi-layer case}\label{sec:multi}

Let us now consider the multi-layer case, that is, the evolution of a neural network with $L+1$ hidden layers (being the previous case, $L = 1$). For the sake of readability, we do it in the case $d = 1$, but the same holds for $d > 1$.   The aim of this section is to introduce and justify all the objects, notably the limit evolution equation and the basis in which such evolution is expressed, for the analogous of Theorem~\ref{thm:main} to hold with $L+1$ hidden layers.  We remark that the following arguments are formal, and that their rigorous justifications can be obtained by the same methods developed in the core of the paper.

Using the notation in subsection~\ref{ssec:settings2}, and dropping the superscript $m$, we now have $\bU\in \R^m$, $\bW^{(\ell)}\in \R^{m\times m}$ for $1\le \ell \le L$, and $\bV\in \R^m$, initialized as 
\[
 U_{j}(0) \sim \mathcal{N}\left(0, 1\right),\qquad W^{(\ell)}_{ij}(0) = 0,\qquad  V_{i}(0) \sim \mathcal{N}\left(0, 1\right).
\]
We also fix $L\in \N$  independent random matrices of size $m\times m$ with independent entries $\mathcal{N}(0, 1)$, $(\bZ^{(\ell)})_{1\le \ell\le L}$.  The neural network is (recall $x\in \R$):
\[
y = h(x, \bU, (\bW^{(\ell)})_{1\le \ell \le L}, \bV) = \left\langle \frac{1}{m} \bV, \prod_{\ell = 1}^L \left(\frac{1}{\sqrt{m}}\bZ^{(\ell)}+\frac{1}{m}\bW^{(\ell)}\right) \bU x\right\rangle.
\]
And the evolution   $(\bU(\kappa),(\bW^{(\ell)}(\kappa))_{1\le \ell\le L},\bV(\kappa))_{\kappa\in \N}$ is a GD (with layer-wise learning rates) on the objective function
\[
F(\bU,(\bW^{(\ell)})_{1\le \ell \le L},\bV) := \int_{\R^d\times \R}\mathcal{L}\left(h(x, \bU, (\bW^{(\ell)})_{1\le \ell \le L}, \bV), y\right)\, d\rho(x, y),
\]
given by
\begin{equation}
\label{eq:training_m_ml}
 \left\{
\begin{split}
\bU(\kappa+1) & = \bU(\kappa)-  {\tau} \prod_{\ell = 1}^L \left[\frac{1}{\sqrt{m}}\bZ^{(\ell)} + \frac{1}{m}\bW^{(\ell)}(\kappa)\right]^\top\bV(\kappa) (\bxi_{\kappa, \tau})^\top,\\
\bW^{(\ell)}(\kappa+1) &= \bW^{(\ell)}(\kappa) - \tau \prod_{i = \ell+1}^{L}\left[\frac{1}{\sqrt{m}}\bZ^{(i)} + \frac{1}{m}\bW^{(i)}(\kappa)\right]^\top\bV(\kappa)  (\bxi_{\kappa, \tau})^\top \\
& \qquad (\bU(\kappa))^\top \prod_{i = 1}^{\ell -1}\left[\frac{1}{\sqrt{m}}\bZ^{(i)} + \frac{1}{m}\bW^{(i)}(\kappa)\right]^\top ,\qquad 1\le \ell \le L,\\
\bV(\kappa+1) &= \bV(\kappa)  -  \tau \prod_{\ell = L}^1 \left[\frac{1}{\sqrt{m}}\bZ^{(\ell)} + \frac{1}{m}\bW^{(\ell)}(\kappa)\right]\bU(\kappa)\bxi_{\kappa, \tau} ,
\end{split}
\right.
\end{equation}
 with $\bxi_{\kappa, \tau} =   \int x\, \mathcal{L}'(h_{\kappa, \tau}(x), y)   d\rho_\kappa(x, y)\in \R$, where we have also denoted $ h_{\kappa, \tau}(x)  =  h(x,  \bU(\kappa),  (\bW^{(\ell)}(\kappa))_{1\le \ell \le L},  \bV(\kappa))$, and we always assume uniformly finite second moments, \eqref{eq:unif_second_moments}. 
 
 In analogy with the three-layer case, we  expect the dynamics to be expressed, up to errors which vanish as $m$ gets large, in a suitable Gaussian basis with certain orthogonality properties, and with an explicit behavior with respect to multiplication by $\bZ^{(\ell)}$. The construction of such a basis (and more precisely, of one basis for each layer $\ell$) is a nontrivial generalization of Theorem~\ref{thm:main_convergence} and it is defined in subection~\ref{ssec:basis} below. We describe now how to obtain the limit dynamics, assuming the existence of such a basis, whose properties are detailed in \eqref{eq:using1_m} and \eqref{eq:using2_m} below.
 
We assume therefore the existence of $L+1$ appropriate orthonormal bases, that we denote
 \[
 \bPsi^{0},\bPsi^1,\dots, \bPsi^{L}, \quad\text{with}\quad \bPsi^\ell = (\bPsi^\ell_1, \bPsi^\ell_2, \bPsi^\ell_3, \dots)\quad\text{for any}\quad \quad 0 \le \ell \le L,
 \]
 such that $\bPsi^\ell \in \R^{m\times\infty}$ is a matrix formed of independent $m$-dimensional Gaussian vectors (as columns), $\bPsi^\ell_i\in \R^{m}$ for all $i\in \N$, with entries $\mathcal{N}(0, 1)$, and that are going to act as the approximate bases for $m < \infty$, satisfying 
 \begin{equation}
 \label{eq:using1_m}
  \frac{1}{m} (\bPsi^{\ell})^\top \bPsi^\ell  = {\rm Id}_{\infty},\quad  \frac{1}{m} (\bPsi^{{\ell}})^\top \bPsi^{\ell'}  = {\rm \0}_{\infty \times \infty}, \qquad 0 \le \ell \neq \ell' \le L
 \end{equation}
  up to errors that vanish as $m \to \infty$ (cf. Theorem~\ref{thm:main_convergence}). Namely, we assume that we can write, up to errors that are of order $O(m^{-\frac12 +\delta})$ for any $\delta > 0$, 
 \begin{equation}
 \label{eq:rep_m}
 \left\{
 \begin{split}
 \bU(\kappa) & = \bPsi^{0}\A(\kappa),\\
  \bW^\ell(\kappa) & = \bPsi^{\ell}\G_\ell (\kappa)(\bPsi^{\ell-1})^\top,\qquad 1\le \ell \le L,\\
   \bV(\kappa) & = \bPsi^{L}\B(\kappa),
  \end{split}
  \right.
 \end{equation}
 for some coefficients $\A, \B \in \R^\infty$, $\G_\ell \in \R^{\infty\times \infty}$ for $1 \le \ell \le L$, initialized as \eqref{eq:init_inf}-\eqref{eq:init_inf2} for $d = 1$ and all $1\le \ell \le L$. Finally, we also assume the following recurrence relationship between bases under multiplication by $\bZ^{(\ell)}$ (cf. subsection~\ref{ssec:recursion}),  
 \begin{equation}
 \label{eq:using2_m}
 \begin{split}\frac{1}{\sqrt{m}}  \bZ^{(\ell)} \bPsi^{\ell-1} &= \bPsi^{\ell} \LL_\ell,\\
 \frac{1}{\sqrt{m}}( \bZ^{(\ell)})^\top \bPsi^\ell &= \bPsi^{\ell-1} \LL_\ell^\top,\qquad 1 \le \ell \le L,
 \end{split}
 \end{equation}
 for some fixed matrices $\LL_\ell\in \R^{\infty\times \infty}$ (cf. equation \eqref{eq:LL_def}). We can then write an evolution for the coefficients $\A, \B,$ and $\G_\ell$, using  \eqref{eq:training_m_ml}-\eqref{eq:using1_m}-\eqref{eq:using2_m} and the representation~\eqref{eq:rep_m}:
 \begin{equation}
\label{eq:training_m_ml2}
 \left\{
\begin{split}
\A(\kappa+1) & = \A(\kappa)-  {\tau} \prod_{\ell = 1}^L (\LL_\ell^\top + \G^\top_\ell(\kappa)) \B(\kappa),\\
\G_{ \ell }(\kappa+1) &= \G_{ \ell }(\kappa) - \tau \prod_{i = \ell+1}^{L}(\LL^\top_i + \G^\top_i(\kappa)) \B(\kappa)  \bxi_{\kappa, \tau}^\top \A^\top(\kappa) \prod_{i = 1}^{\ell -1}(\LL_i^\top+\G^\top_i(\kappa)),\\
\B(\kappa+1) &= \B(\kappa)  -  \tau \prod_{\ell = L}^1 (\LL_\ell+\G_\ell(\kappa)) \A(\kappa)\bxi_{\kappa, \tau} ,
\end{split}
\right.
\end{equation}
for $1\le \ell \le L$, with 
\begin{align*}
\chi_{\kappa, \tau}(x) &= \chi(x, \A(\kappa), (\G_\ell(\kappa))_{1\le \ell \le L}, \B(\kappa)), \\
\bxi_{\kappa, \tau} & = \int x\mathcal{L}'(\chi_{\kappa, \tau}(x), y) d\rho_\kappa(x, y)\in \R. %,\qquad \chi(x, \A, \G, \B) = \B^\top[\LL + \G]\A x.
\end{align*}
When $\rho_\kappa =\rho$ for all $\kappa\in \N$, this recursion is exactly the GD on the (deterministic) objective function $\mathcal{E}$ defined by
\begin{align}\label{eq:limit-objective2}
\mathcal{E}(\A,(\G_\ell)_{1\le \ell \le L},\B) = \int \mathcal{L}\left(\B^\top\prod_{\ell = L}^1 (\LL_\ell+\bG_\ell)\A x), y\right)d\rho(x,y),
\end{align}
and the linear predictor of the neural network is given by 
\[
\A^\top\prod_{\ell = 1}^L (\LL^\top_\ell+\bG^\top_\ell)\B,
\]
up to errors that disappear as $m\to \infty$. Thus, the description of the linear neural network in the general multi-layered case, \eqref{eq:training_m_ml2}, is reduced to finding bases such that \eqref{eq:using1_m} and \eqref{eq:using2_m} hold, up to errors (which is precisely what we did in Section~\ref{sec:basis} above).

\subsection{The choice of the bases}\label{ssec:basis}

Given $L \in \N$ and $0\le \ell \le L$, let us define the following set of finite sequences:
\[
\mathcal{S}^L(\ell) := \left\{(s_0, s_1, s_2, \dots, s_M) : s_0\in \{0, L\}, s_M = \ell, s_i \in \{0,\dots,L\}, |s_i-s_{i-1}| = 1 \right\},
\]
that is, $\mathcal{S}^L(\ell)$ is the set of finite sequences of numbers belonging to $\{0,\dots, L\}$, starting at $0$ or $L$, finishing at $\ell$, and such that each element of the sequence is obtained by adding or subtracting 1 to the previous element (in particular, if $s_0 = 0$, $s_1 = 1$ necessarily). This set is going to be, for each $0\le \ell\le L$, our index set for the basis $\bPsi^\ell$.  For example, when $L = 1$, the sequences in $\mathcal{S}^1(0)$ (and analogously in $\mathcal{S}^1(1)$) are just of the form $0101...0$ or $1010...0$, and can be identified with their length. This is the reason why the index set in the case $L=1$ is just given by the natural numbers, which was the case in Section~\ref{sec:basis}. 

We therefore consider $\bPsi^\ell$ to have as columns the elements $\bPsi^\ell_s$ for $s\in \mathcal{S}^L(\ell)$, and we denote it,
\[
\bPsi^\ell = (\bPsi^\ell_s)_{s\in \mathcal{S}^L(\ell)},\qquad 0 \le \ell \le L, 
\]
where we still need to define what $\bPsi^\ell_s$ is for a given $s\in \mathcal{S}^L(\ell)$. To do so, for notational convenience, given the matrices $\bZ^{(\ell)}$ for $1\le \ell \le L$, we denote
\[
\bZ^{\ell-1, \ell} := (\bZ^{(\ell)})^\top\qquad \text{and}\qquad \bZ^{\ell, \ell-1} := \bZ^{(\ell)}.
\]
Moreover, we let $\bPsi^0_{0}$ and $\bPsi^L_{L}$ be two fixed independent Gaussian vectors of size $m$ (that is, those associated to the sequences $\{0\}$ and $\{L\}$). 

Then, given $s\in \mathcal{S}^L(\ell)$ of length $M+1$, $s = (s_0,\dots,s_M)$, we define 
\begin{equation}
\label{eqn:psinew}\bPsi^\ell_s := m^{-M/2} \sum_{(i_0,\dots, i_{M})\in \mathcal{I}(s, m)} \left( \prod_{j=1}^{M} \bZ_{i_j, i_{j-1}}^{s_{j}, s_{j-1}} \right)(\bPsi^{s_0}_{s_0})_{i_0},
\end{equation}
where $\mathcal{I}(s, m)$ is the set of indices $(i_0, \dots, i_M)$ with $i_j \in \{0,\dots,m\}$  such that $(i_j, s_j) \neq (i_{k}, s_k)$ for all $1\le j\neq k\le M$.  In other words, the main novelty of the current definition with respect to the corresponding definition \eqref{eq:JKdef2} for $L=1$ lies in the fact that the basis is parametrized by an element $s \in \mathcal{S}^L(\ell)$, which identifies a fixed sequence of consecutive layers. Once the sequence is fixed, the sum in \eqref{eqn:psinew} runs over all possible loopless choices of one element between $1,...,m$ in each of the layers signposted by $s$. 

%The heuristic reason for such a definition follows from an analogous computation to \eqref{eq:mult_odd}-\eqref{eq:mult_even}.
% but we keep only those elements of the sum that have no loops (and we rescale by the appropriate size, where the size-preserving objects are $m^{-1/2} \bZ^m$). 

Formally, we obtain orthonormal bases  in the sense \eqref{eq:using1_m} (as in Proposition~\ref{prop:orthonormality}), and the relationships in \eqref{eq:using2_m} are of the form 
\begin{equation}
\label{eq:rel1}
\frac{1}{\sqrt{m}} \bZ^{(\ell)}\bPsi^{\ell-1}_s = \left\{
\begin{array}{ll}
\bPsi^\ell_{(s, \ell)}&\quad\text{ if }\quad s = (s', \ell-2, \ell-1),\\
\bPsi^{\ell}_{(s', \ell)}+\bPsi^\ell_{(s, \ell)}&\quad\text{ if }\quad s = (s', \ell, \ell-1),
\end{array}
\right.
\end{equation}
and 
\begin{equation}
\label{eq:rel2}
\frac{1}{\sqrt{m}} (\bZ^{(\ell)})^\top \bPsi^{\ell}_s = \left\{
\begin{array}{ll}
\bPsi^{\ell-1}_{(s, \ell-1)}&\quad\text{ if }\quad s = (s', \ell+1, \ell),\\
\bPsi^{\ell-1}_{(s', \ell-1)}+\bPsi^{\ell-1}_{(s, \ell-1)}&\quad\text{ if }\quad s = (s', \ell-1, \ell),
\end{array}
\right.
\end{equation}
 for $1 \le \ell \le L$. 
 
 \subsection{The case $L = 2$}
 
 In the case $L = 2$ (that is, a four layers neural network, or a neural network  with three hidden layers) we have a more explicit expression. In this case,  any element $s\in \mathcal{S}^2(\ell)$ is of the form 
 \[
 (s_0, 1, s_2, 1, s_4, 1, s_6, 1, s_8, 1, \dots),\quad\dots\quad s_{2i} \in \{0, 2\},
 \]
 and therefore, we can identify any element   $s$ in  $\mathcal{S}^2(0)$, $\mathcal{S}^2(1)$, or $\mathcal{S}^2(2)$, with a natural number $N(s)$, seeing it as a binary representation. Thus, we associate 
 \[
 \begin{split}
 \mathcal{S}^2(0)\ni s &\mapsto  N_0(s) := 2^{\sigma}+\sum_{i=1}^\sigma 2^{i-2} s_{2 (\sigma-i)}\\
  \mathcal{S}^2(1)\ni s &\mapsto  N_1(s) := 2^{\sigma+1}+\sum_{i=0}^\sigma 2^{i-1} s_{2 (\sigma-i)}\\
   \mathcal{S}^2(2) \ni s&\mapsto N_2(s) := 2^{\sigma}+\sum_{i=1}^\sigma 2^{i-2} s_{2 (\sigma-i)}
 \end{split}
 \]
 where we have denoted $\sigma = \lfloor M/2 \rfloor$ for $s = (s_0,\dots,s_M)$.  With this indexing, we can obtain more explicit relations \eqref{eq:rel1}-\eqref{eq:rel2}, since we now have that $\bPsi^0$, $\bPsi^1$, and $\bPsi^2$ can   be indexed by the natural numbers. That is, as an abuse of notation we denote
 \[
\begin{array}{ll}
 \bPsi^i_j = \bPsi^i_s& \quad \text{if}\quad N_i(s) = j,\quad\text{for}\quad i = 0, 1, 2,
 \end{array}
 \]
 which is well-defined for any $j \ge 2$.

The relations \eqref{eq:rel1}-\eqref{eq:rel2} correspond to
 \begin{equation}
\label{eq:rel1_2}
\frac{1}{\sqrt{m}} \bZ^{(1)}\bPsi^{0}_j = \bPsi^{1}_{j}+\bPsi^1_{2j}, 
\end{equation}
 \begin{equation}
\label{eq:rel1_25}
\frac{1}{\sqrt{m}} (\bZ^{(2)})^\top\bPsi^{2}_j = \bPsi^{1}_{j}+\bPsi^1_{2j+1}, 
\end{equation}
and 
\begin{equation}
\label{eq:rel2_2}
\frac{1}{\sqrt{m}} (\bZ^{(1)})^\top \bPsi^{1}_j = \left\{
\begin{array}{ll}
\bPsi^{0}_{j}&\quad\text{ if $j$ is odd},\\
\bPsi^{0}_{j}+\bPsi^{0}_{j/2}&\quad\text{ if $j$ is even},
\end{array}
\right.
\end{equation}
\begin{equation}
\label{eq:rel2_25}
\frac{1}{\sqrt{m}} \bZ^{(2)}\bPsi^{1}_j = \left\{
\begin{array}{ll}
\bPsi^2_{j}&\quad\text{ if $j$ is even},\\
\bPsi^{2}_{j}+\bPsi^2_{(j-1)/2}&\quad\text{ if $j$ is odd}.
\end{array}
\right.
\end{equation}

Thanks to \eqref{eq:rel1_2}-\eqref{eq:rel1_25}-\eqref{eq:rel2_2}\eqref{eq:rel2_25}, the matrices $\LL_1$ and $\LL_2$ in \eqref{eq:training_m_ml2} can be determined, which are the only missing unknowns to be able to obtain an evolution of the system \eqref{eq:training_m_ml2}:
\[
 (\Lambda_1)_{ij} = 
\left\{
\begin{array}{ll}
1 & \quad\text{if $i = j$ or $2i = j$},\\
0 & \quad\text{otherwise},
\end{array}
\right.
\quad\text{and}\quad (\Lambda_2)_{ij} = 
\left\{
\begin{array}{ll}
1 & \quad\text{if $i = j$ or $2j+1 = i$},\\
0 & \quad\text{otherwise},
\end{array}
\right.
\]
that is, 
\[
\LL_1 = 
\begin{pmatrix}
1 & 1 & 0 & 0 & 0 & 0 & 0 & 0 & 0 &\\
0 & 1 & 0 & 1 &0 & 0 & 0 & 0 & 0 &  \dots\\
0 & 0 & 1 & 0 & 0 & 1 & 0 & 0 & 0 & \\
0 & 0 & 0 & 1 & 0 & 0 & 0 & 1 & 0 & \\
 &  &  &  & \vdots &  &  &  &  & \ddots
\end{pmatrix},
\]
and 
\[
\LL_2^\top = 
\begin{pmatrix}
1 & 0 & 1 & 0 & 0 & 0 & 0 & 0 & 0 &\\
0 & 1 & 0 & 0 &1 & 0 & 0 & 0 & 0 &  \dots\\
0 & 0 & 1 & 0 & 0 & 0 & 1 & 0 & 0 & \\
0 & 0 & 0 & 1 & 0 & 0 & 0 & 0 & 1 & \\
 &  &  &  & \vdots &  &  &  &  & \ddots
\end{pmatrix}.
\]
\bibliographystyle{plain}
\bibliography{LNN}

\end{document}